\documentclass{article}

\usepackage{times}

\usepackage{graphicx}

\usepackage{multirow}
\usepackage{amsmath}
\usepackage{amsthm}
\usepackage{amssymb,latexsym,color}
\usepackage{mdwlist}
\usepackage{todonotes}

\usepackage{enumitem}

\usepackage{url}

\usepackage{wrapfig}

\usepackage{algorithm}
\usepackage[noend]{algpseudocode}
\usepackage{marvosym}
\usepackage{fullpage}
\usepackage{natbib}
\usepackage{graphicx}
\graphicspath{ {Figures/} }
\usepackage{subcaption}

\newcommand{\R}{\mathbb{R}}
\newcommand{\sgn}[1]{\mbox{sgn}(#1)}
\renewcommand{\vec}[1]{\mathbf{#1}}
\newcommand{\poly}{\mathrm{poly}}

\def\a{{\bf a}}
\def\g{{\bf g}}
\def\x{{\bf x}}
\def\y{{\bf y}}
\def\w{{\bf w}}
\def\v{{\bf v}}
\def\E{\mathbb{E}}
\def\rrow{r_\mathrm{row}}

\newcommand{\err}{\ensuremath{\mathrm{err}}}
\newcommand{\setX}{\Omega}
\newcommand{\setI}{\mathcal{I}}
\newcommand{\OPT}{\ensuremath{\mathrm{OPT}}}
\newcommand{\setF}{\mathcal{F}}
\newcommand{\setJ}{\mathcal{J}}

\DeclareMathOperator*{\argmin}{argmin}
\newtheorem*{lemma*}{Lemma}
\newtheorem{lemma}{Lemma}
\newtheorem*{theorem*}{Theorem}
\newtheorem{theorem}{Theorem}

\newtheorem{corollary}{Corollary}

\newtheorem{definition}{Definition}

\usepackage{mathtools}

\title{ \bf
The ZipML Framework for Training Models with End-to-End Low Precision:\\
The Cans, the Cannots, and a Little Bit of Deep Learning
}

\author{
\small Hantian Zhang$^\ast$$^\dag$~~~~Jerry Li$^\ast$$^\sharp$~~~~Kaan Kara$^\dag$\\
\small Dan Alistarh$^\dag$~~~~Ji Liu$^\ddag$~~~~Ce Zhang$^\dag$\\
\small $^\dag$ ETH Zurich\\
\small \{hantian.zhang, kaan.kara, dan.alistarh, ce.zhang\}@inf.ethz.ch\\
\small $^\ddag$ University of Rochester\\
\small jliu@cs.rochester.edu\\
\small $^\sharp$ Massachusetts Institute of Technology\\
\small jerryzli@mit.edu \\
\small $^\ast$ Equal Contribution
}
\date{}

\begin{document}
\maketitle

\begin{abstract}

Recently there has been significant interest in training 
machine-learning models at low precision: by reducing 
precision, one can reduce computation and communication by one order of magnitude. 
We examine training at reduced precision, both from a theoretical and practical 
perspective, and ask: 
{\em is it possible to \emph{train} models at end-to-end low 
precision with \emph{provable} guarantees? Can this 
lead to consistent order-of-magnitude speedups?}
We present a framework called ZipML to answer these questions.
For linear models, the answer is yes. We develop a simple 
framework based on one simple but novel strategy called double sampling. 
Our framework is able 
to execute training at low precision with no bias, 
guaranteeing convergence, whereas naive quantization 
would introduce significant bias. We validate our framework   
across a range of applications, and show that it enables an 
FPGA prototype that is up to $6.5\times$ faster 
than an implementation using full 32-bit precision.
We further develop a variance-optimal 
stochastic quantization
strategy and show that 
it can make a significant difference in a variety of settings. 
When applied to linear models together with 
double sampling, we save up to another 
$1.7\times$ in data movement compared with the uniform quantization.
When
training deep networks with quantized models, 
we achieve higher accuracy than the state-of-the-art XNOR-Net. 
Finally, we extend our framework through approximation to non-linear 
models, such as SVM. We show that, although using low-precision data induces bias, 
we can appropriately 
bound and control the bias. We find in practice {\em 8-bit} 
precision is often sufficient to converge to the correct solution. 
Interestingly, however, in practice we notice that our framework does not always outperform the naive rounding approach. We discuss this negative result in detail.

\end{abstract}

\begin{figure}[t]
\centering
\includegraphics[width=0.5\textwidth]{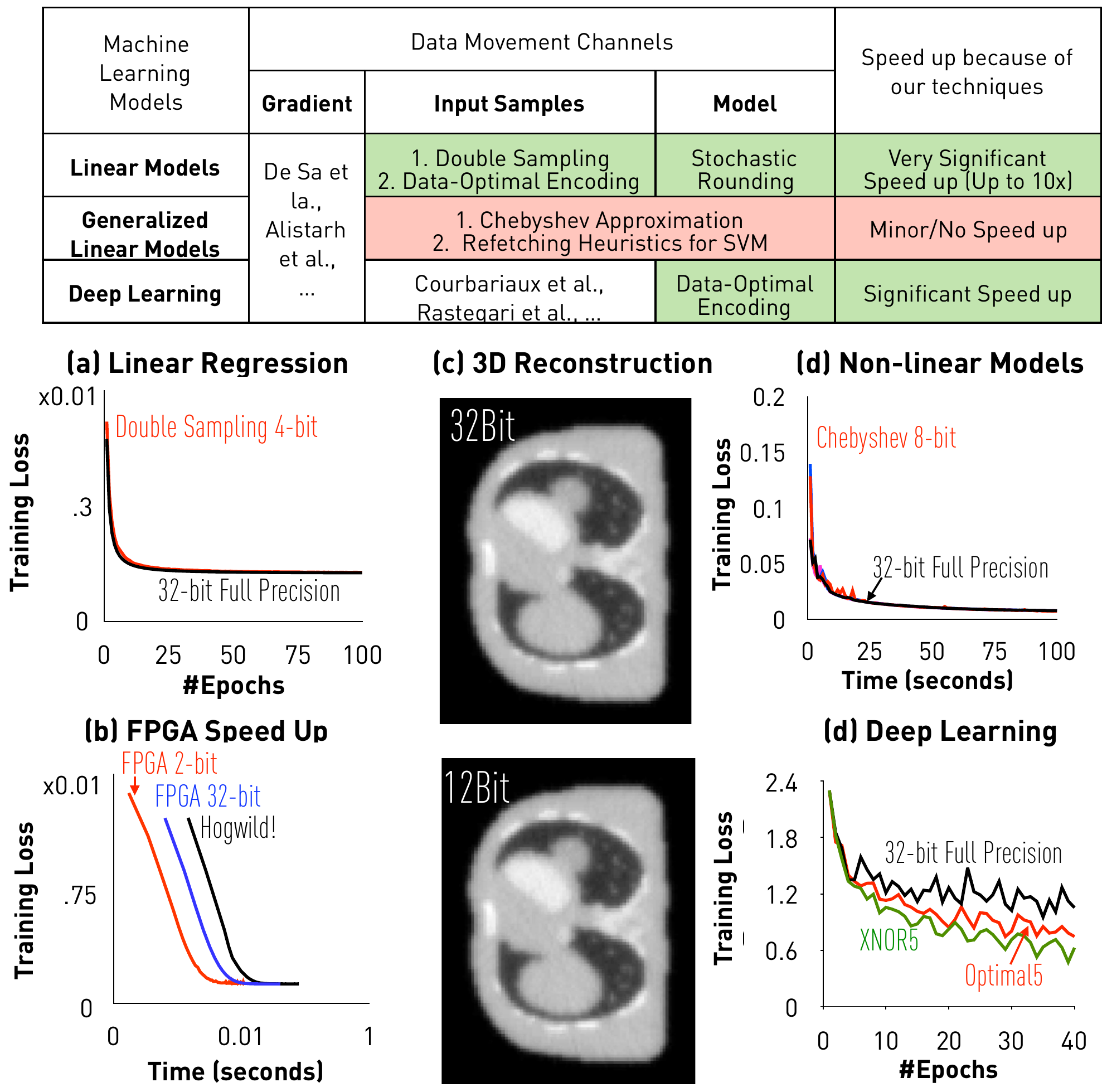}    
\caption{Overview of theoretical results and
highlights of empirical results. See
Introduction for details.}
\label{fig:highlight}
\end{figure}

\section{Introduction}

The computational cost and power consumption of today's machine learning systems are often driven by data movement, and by the precision of computation. 
In our experience, in applications such as tomographic reconstruction, anomaly detection in mobile sensor networks,
and compressive sensing, the overhead of transmitting the data samples can be massive, 
and hence performance can hinge on reducing the precision of data representation and 
associated computation. 
A similar trend is observed in deep learning, where impressive progress has been reported with systems 
using end-to-end reduced-precision representations~\cite{hubara2016quantized,
rastegari2016xnor,zhou2016dorefa,miyashita2016convolutional}. 
In this context, the motivating question behind our work is:  {\em When training general machine learning models,
can we lower the precision of data representation,
communication, and computation, while maintaining provable guarantees?}
 
In this paper, we develop a general 
framework to answer this question, and
present both positive and negative results
 obtained in the context of this framework. 
 Figure~\ref{fig:highlight} encapsulates our results: 
(a) for linear models, we are able to lower the precision of both computation and communication, including input samples, gradients, and model, by up to $16$ times, while still providing rigorous theoretical guarantees; 
(b) our FPGA implementation of this framework achieves up to $6.5\times$ speedup compared with
a 32-bit FPGA implementation, or with a 10-core CPU running Hogwild!;  
(c) we are able to decrease data movement by $2.7\times$ for
tomographic reconstruction, while obtaining a negligible quality decrease. 
Elements of our framework generalize to (d) non-linear models and (e) model compression for training deep learning models. 
In the following, we describe our technical contributions in more detail.

\subsection{Summary of Technical Contributions}

We consider the following problem in training generalized linear models: 
\begin{align}
\min_{\x}:\quad {1\over 2K}\sum_{k=1}^K l(\a_k^\top \x, b_k)^2 + R(\x),
\label{eqn:leastsquares}
\end{align}
where $l(\cdot,\cdot)$ is a loss function and $R$ is a regularization term that could be $\ell_1$ norm, $\ell_2$ norm, or even an indicator function representing the constraint. 
The gradient at the sample $(\a_k, b_k)$ is: 
\[
\g_k := \a_k \frac{\partial l(\a_k^\top \x, b_k)}{\partial \a_k^\top \x} .
\]
We denote the problem dimension by $n$. 
We consider the properties of the algorithm when a lossy compression scheme is applied to the data (samples), 
gradient, and model, to reduce the communication cost of the algorithm---that is, we consider quantization functions $Q_g$, $Q_m$, and $Q_s$ for gradient, model, and samples, respectively, in the gradient update:
\begin{align}
\x_{t + 1} \leftarrow \text{prox}_{\gamma R(\cdot)}\left(\x_t - \gamma Q_g(\g_k (Q_m(\x_t), Q_s(\vec{a}_t)))\right),
\label{eq:proxupdate}
\end{align}
where the proximal operator is defined as
\[
\text{prox}_{\gamma R(\cdot)}(\y) =\argmin_{\x} {1\over 2}\|\x-\y\|^2 + \gamma R(\x).
\]

\paragraph{Our Results.} We summarize our results as follows. The {\bf (+)}
sign denotes a ``positive result,'' where we achieve
significant practical speedup; it is {\bf (--)} otherwise.

\paragraph{(+) Linear Models.} When $l(\cdot,\cdot)$ is 
the least squares loss, we first notice that
simply doing stochastic quantization of data samples  
(i.e., $Q_s$) introduces bias of the gradient
estimator and therefore SGD would converge
to a different solution. We propose a simple
solution to this problem by introducing a
{\em double sampling} strategy
$\tilde{Q}_s$ that uses multiple samples to
eliminate the correlation of samples introduced
by the non-linearity of the gradient. We
analyze the additional variance introduced
by double sampling, and find that its impact is \emph{negligible in terms of convergence time} as long as the 
number of bits used to store a quantized sample is at least $\Theta( \log n / \sigma )$, 
where $\sigma^2$ is the variance of the standard stochastic gradient. 
This implies that the 32-bit precision may be excessive for many practical scenarios. 

We build on this result to obtain an \emph{end-to-end quantization} strategy
for linear models, which compresses all data movements. 
For certain settings of parameters, end-to-end quantization adds as little as a \emph{constant factor} to the variance of the entire process. 

\paragraph{(+) Optimal Quantization and Extension to Deep Learning.}
We then focus on reducing the variance of  
stochastic quantization. We notice that different methods for setting the quantization points have different variances---the standard uniformly-distributed quantization strategy is far from optimal in many settings.
We formulate this as an independent optimization problem, and solve it optimally with 
an efficient dynamic programming algorithm 
that only needs to scan the data in a single pass.
When applied to linear models, this optimal 
strategy can save up to $1.6\times$ communication
compared with the uniform strategy.

We perform an analysis of the optimal quantizations for various settings, and observe that the uniform quantization approach
popularly used by state-of-the-art end-to-end
low-precision deep learning training systems
when more than 1 bit is used is suboptimal.
We apply optimal quantization to 
model quantization and show that, with one
standard neural network, we outperform the
uniform quantization used by XNOR-Net and a
range of other recent approaches. This
is related, but different, to recent work 
on model compression for inference~\cite{Han:2016:ICLR}. 
To the best of our knowledge, this is the first time such optimal quantization strategies have been applied to training. 

\paragraph{(--) Non-Linear Models.} We extend our
results to non-linear models, such as SVM. We can stretch our multiple-sampling strategy to provide 
unbiased estimators for any polynomials, at the cost of increased variance. 
Building further, we employ Chebyshev polynomials to   
approximate the gradient of \emph{arbitrary smooth loss functions} within arbitrarily low bias, 
and to provide bounds on the error of an SGD solution obtained from low-precision samples. 

Further, we examine whether this approach can be applied to non-smooth loss functions, such as SVM. 
We find that the machinery described above does not apply, for fundamental reasons. 
We use ideas from streaming and dimensionality reduction to develop a variant that is provably correct for non-smooth loss functions. 
We can show that, under reasonable assumptions, the added communication cost of supporting non-smooth functions is negligible. 

In practice, using this technique we are
able to go as low as 8-bit precision for SVM and logistic regression. 
However, we notice that the straw man approach, which applies naive stochastic rounding over the input data to just 8-bit precision, converges to similar results, 
without the added complexity. 
This negative result is explained by the fact that, to approximate non-linearities such as the step function or the sigmoid well, our framework needs both high degree Chebyshev polynomials and relatively large samples. 

\begin{figure}[t]
\centering   
\includegraphics[scale=0.4]{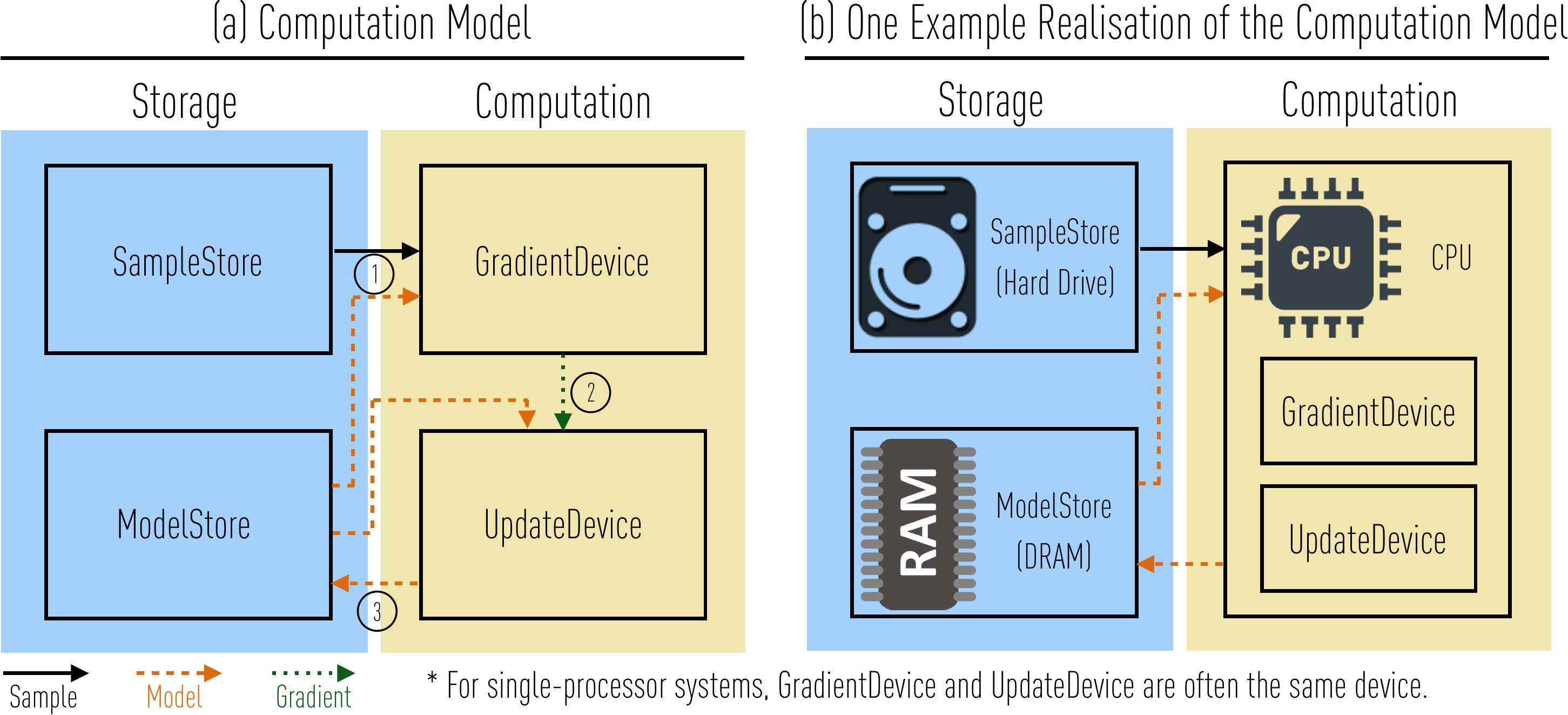}
\caption{(a) A Schematic Representation of the Computation Model and (b) An Example Realisation
of the Computation Model. Three types of
data, namely (1) sample, (2) model, and (3)
gradient, moves in the system in three
steps as illustrated in (a). Given
different parameters of the computation model,
such as computational power and memory bandwidth, the system bottleneck may
vary. For example, in 
realisation (b) having a hard drive, DRAM, and a
modern CPU, it is likely that the  bottleneck when training 
a dense generalized linear model is the
memory bandwidth between SampleStore
and GradientDevice.}
\label{fig:model}
\end{figure}

\section{Linear Models}

In this section, we focus on linear models with possibly non-smooth regularization. We have labeled data points $(\a_1, b_1), (\a_2, b_2), \ldots, (\a_K, b_K) \in \R^n \times \R$, and our goal is to minimize the function
\begin{align}
F(\x) = \underbrace{\frac{1}{K} \sum_{k = 1}^K \| \a_k^\top \x - b_k \|_2^2}_{=: f(\x)} + R(\x) \; ,
\label{eq:linear}
\end{align}
i.e., minimize the empirical least squares loss plus a non-smooth regularization $R(\cdot)$ (e.g., $\ell_1$ norm, $\ell_2$ norm, and constraint indicator function). SGD is a popular approach for solving large-scale machine learning problems. It works as follows: at step $\x_t$, given an unbiased gradient estimator $\g_t$, that is, 
$\E(\g_t) = \nabla f(\x_t),$
we update $\x_{t+1}$ by
\[
\x_{t+1} = \text{prox}_{\gamma_t R(\cdot)}\left( \x_t - \gamma_t \g_t\right),
\]
where $\gamma_t$ is the predefined step length. SGD guarantees the following convergence property:
\begin{theorem}\label{thm:sgd-conv}[e.g., \cite{bubeck2015convex}, Theorem 6.3]
Let the sequence $\{\x_t\}_{t=1}^T$ be bounded. Appropriately choosing the steplength,
we have the following convergence rate for \eqref{eq:linear}:
\begin{equation}
F\left(\frac{1}{T} \sum_{t = 0}^T \x_t\right) - \min_{\x}F(\x) \leq \Theta\left({1\over T} + {\sigma \over \sqrt{T}}\right) 
\label{eq:sgd-conv}
\end{equation}
where $\sigma$ is the upper bound of the mean variance 
\[
\sigma^2 \geq {1\over T}\sum_{t=1}^T \E\|\g_t - \nabla f(\x_t)\|^2. 
\]
\end{theorem} 
There are three key requirements for SGD to converge:
\begin{enumerate}[leftmargin=*, noitemsep]
\item Computing stochastic gradient $\g_t$ is cheap;
\item The stochastic gradient $\g_t$ should be unbiased;
\item The stochastic gradient variance $\sigma$ dominates the convergence efficiency, so it needs to be controlled appropriately.
\end{enumerate}
The common choice is to uniformly select one sample:
\begin{align}
\g_t = \g_t^{(full)} := \a_{\pi(t)} (\a_{\pi(t)}^\top \x - b_{\pi(t)}).
\label{eq:sgfull}
\end{align} 
($\pi(t)$ is a uniformly random integer from $1$ to $K$). We abuse the notation and let $\a_t = \a_{\pi(t)}$. Note that $\g_t^{(full)}$ is an unbiased estimator $\E [\g_t^{(full)}] = \nabla f(\x_t)$. Although it has received success in many applications, 
if the precision of sample $\a_{t}$ can be further decreased,
we can save potentially one order of magnitude bandwidth
of reading $\a_{t}$ (e.g., in sensor networks) and the associated computation (e.g.,
each register can hold more numbers). 
This motivates us to use low-precision sample points to train the model. The following will introduce the proposed low-precision SGD framework by meeting all three factors for SGD.

\subsection{Bandwidth-Efficient Stochastic Quantization} 

We propose to use stochastic quantization to generate a low-precision version of an arbitrary vector $\v$ in the following 
way. Given a vector
$\v$, let $M(\v)$ be a scaling factor such that $-1 \le \v/M(\v) \le 1$. Without loss of generality, let $M(\v)=||\v||_2$. We partition the interval $[-1, 1]$ using $s+1$ separators: $-1 = l_0 \le l_1 ... \le l_{s} = 1$; for each number $v$ in $\v/M(\v)$, we 
quantize it to one of two nearest separators: $l_i \le v \le l_{i+1}$. We denote the \emph{stochastic quantization} function by $Q(\v, s)$ and choose the probability of quantizing to different separators such that $\E[Q(\v, s)] = \v$. We use $Q(\v)$ when $s$ is not relevant.

\subsection{Double Sampling for Unbiased Stochastic Gradient}

\begin{wrapfigure}{r}{0.23\textwidth}
  \begin{center}
    \includegraphics[width=0.23\textwidth]{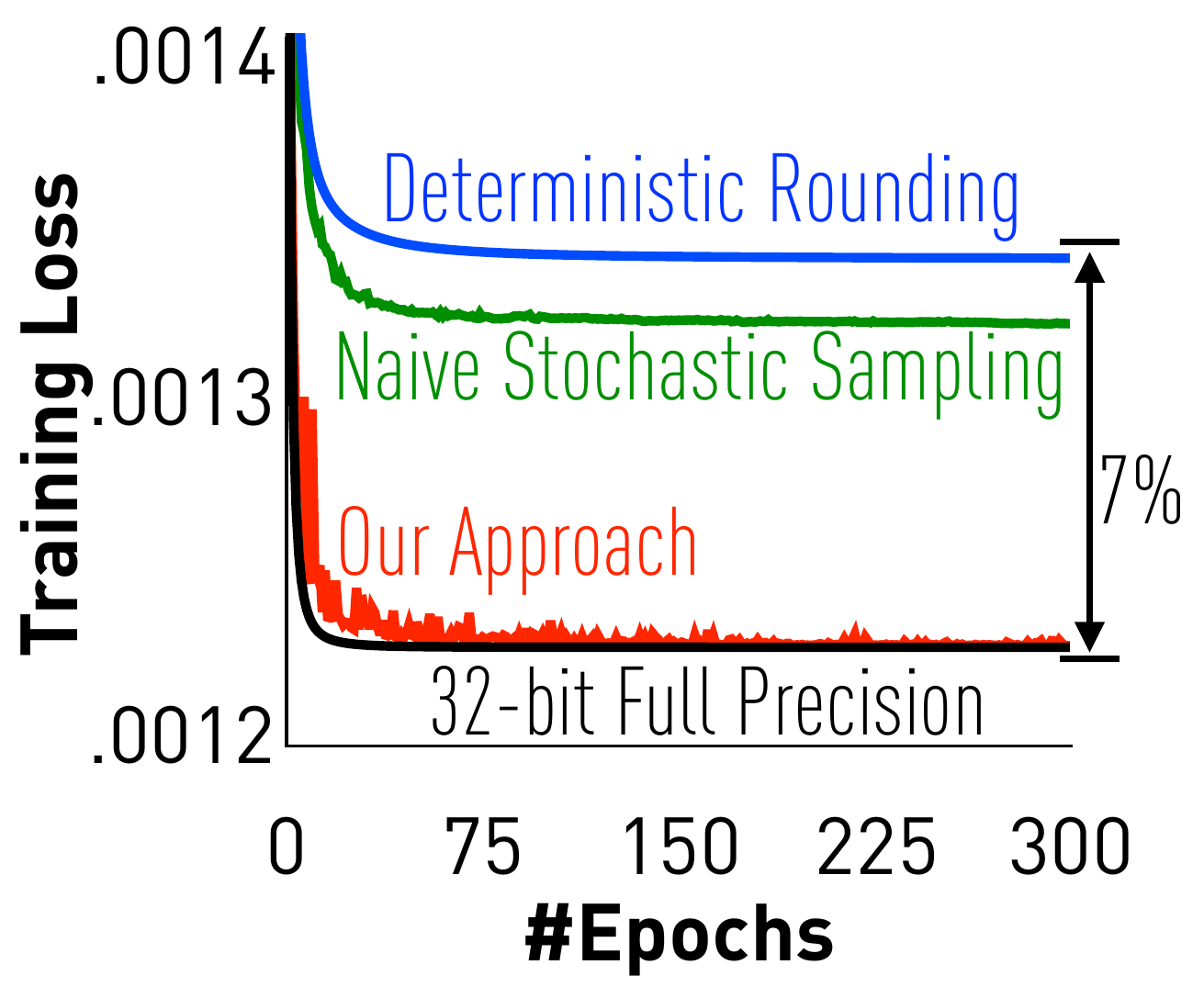}
  \end{center}
  \label{fig:gap}
\end{wrapfigure}
The naive way to use low-precision samples $\hat{\a}_t := Q(\a_t)$ is 
\[
\hat{\g}_t := \hat{\a}_t \hat{\a}_t^\top \x - \hat{\a}_t b_t.
\]
However, \emph{the naive approach does not work} (that is, it does not guarantee convergence), because it is biased: 
\[
\E[\hat{\g}_t] := \a_t \a_t^\top \x - \a_t b_t + D_{\a} \x, 
\]
where $D_{\a}$ is diagonal and its $i$th diagonal element is 
\[
\E[ Q(\a_i)^2 ] - \a_i^2.
\]

Since $D_{\a}$ is non-zero, we obtain a \emph{biased} estimator of the gradient, so the iteration is unlikely to converge. 
The figure on the right illustrates the bias caused by a non-zero $D_{\a}$. In fact, it is easy to see that in instances where the minimizer $\x$ is large and gradients become small, we will simply diverge. 

We now present a simple method to fix the biased gradient estimator. We generate two independent random quantizations and revise the gradient:
\begin{align}
\g_t := Q_1 (\a_t) (Q_2 (\a_t)^\top \x - b_t) \; .
\label{eq:double}
\end{align}
This gives us an unbiased estimator of the gradient. 

\paragraph*{Overhead of Storing Samples.}
The reader may have noticed that one implication of double sampling is the overhead of sending
two samples instead of one. We note that this will not introduce $2\times$
overhead in terms of data communication. Instead, we start from the observation that the two samples can  
differ by at most one bit. For example, to quantize the number 0.7 to either 0 or 1. Our strategy is to first store the smallest number of the interval (here 0), and then for each sample, send out 1 bit to represent whether this sample is at the lower marker (0) or the upper marker (1). Under this procedure, once we store the base quantization level, we will need one extra bit for each additional sample. 
More generally, since samples
are used symmetrically, we only need to send a number representing the number€œ of times  the lower quantization level has been chosen among all the sampling trials. 
Thus, sending $k$ samples only requires $\log_2 k$ more bits.
\subsection{Variance Reduction}

From Theorem~\ref{thm:sgd-conv}, the mean variance ${1\over T}\sum_{t}\E\|\g_t - \nabla f(\x)\|^2$ will dominate the convergence efficiency. It is not hard to see that the variance of the double sampling based stochastic gradient in \eqref{eq:double} can be decomposed into
\begin{align}
\begin{split}
\E\|\g_t - \nabla f(\x_t)\|^2 & \leq \E \|\g_t^{(full)}- \nabla f(\x_t)\|^2 
\\
&+ \E \|\g_t - \g_t^{(full)}\|^2.
\label{eqn:varbound}
\end{split}
\end{align}
The first term is from the full stochastic gradient, which can be reduced by using strategies such as mini-batch, weight sampling, and so on. Thus, reducing the first term is an orthogonal issue for this paper. 
Rather, we are interested in the second term, which is the additional cost of using low-precision samples. All strategies for reducing the variance of the first term can seamlessly combine with the approach of this paper. 
The additional cost can be bounded by the following lemma.
\begin{lemma} 
The stochastic gradient variance using double sampling in \eqref{eq:double} $\E\|\g_t - \g_t^{(full)}\|^2$ can be bounded by
\begin{align*}
&\Theta\left(\mathcal{TV}(\a_t) (\mathcal{TV}(\a_t)\|\x\odot \x\| + \|\a_t^\top \x\|^2 + \|\x\odot \x\|\|\a_t\|^2)\right),
\end{align*}
where $\mathcal{TV}(\a_t) := \E\|Q(\a_t) - \a_t\|^2$ and $\odot$ denotes the element product.
\end{lemma}
Thus, minimizing $\mathcal{TV}(\a_t)$ is key to reducing variance. 

\paragraph{Uniform quantization.} It makes intuitive sense that, the more levels of quantization, the lower the variance. The following  makes this quantitative dependence precise. 

\begin{lemma}
\label{lem:quant-facts} [\cite{Alistarh:2016:ArXiv}]
Assume that quantization levels are uniformly distributed. For any vector $\vec{v} \in \R^n$, we have that $\E [Q (\vec{v},s)] = \vec{v}$. Further, the variance of uniform quantization with $s$ levels is bounded by
\[
\mathcal{TV}_s(\v):=\E [\| Q (\vec{v},s) - \v\|_2^2] \leq \min( n/s^2,\sqrt{n}/s)) \| \vec{v} \|_2^2. \; .
\]
\end{lemma} 

Together with other results, it suggests the stochastic gradient variance of using double sampling is bounded by
\[
\E\|\g_t - \nabla f(\x_t)\|^2 \leq \sigma^2_{(full)} + \Theta \left( {n /s^2} \right),
\]
where $\sigma^2_{(full)} \geq \E \|\g_t^{(full)} - \nabla f(\x)\|^2$ is the upper bound of using the full stochastic gradient, assuming that $\x$ and all $\a_k$'s are bounded. Because the number of quantization levels $s$ is exponential to the number of bits we use to quantize, to ensure that these two terms are comparable (using a low-precision sample does not degrade the convergence rate), the number of bits only needs to be greater than $\Theta (\log n /\sigma_{(full)})$. Even for linear models with millions
of features, 32 bits is likely to be  ``overkill.''

\section{Optimal Quantization Strategy for Reducing Variance} \label{sec:optimal}


In the previous section, we have assumed uniformly distributed quantization points.  
We now investigate the choice of quantization points and present an optimal strategy to minimize the quantization variance term $\mathcal{TV}(\a_t)$.

\paragraph*{Problem Setting.}
Assume a set of real numbers $\Omega = \{x_1, \ldots, x_N\}$ with cardinality $N$. WLOG, assume that all numbers are in $[0, 1]$ and that $x_1 \leq \ldots \leq x_N$. 

The goal is to partition $\setI = \{I_j\}_{j = 1}^s$ of $[0, 1]$ into $s$ disjoint intervals, so that if we randomly quantize every $x \in I_j$ to an endpoint of $I_j$, the variance is minimal over all possible partitions of $[0, 1]$ into $s$ intervals.
Formally:
\begin{align}
\nonumber \min_{\setI: |\setI| = s} \quad & \mathcal{MV}(\setI) := {1\over N}\sum_{j = 1}^s \sum_{x_i \in I_j} \err(x_i, I_j)\\
\text{s.t.}\quad & \bigcup_{j = 1}^s I_j = [0, 1],\quad I_j\cap l_k = \emptyset~\text{for $k\neq j$},
\label{eq:opt_Q}
\end{align}
where $\err (x, I) = (b - x) (x - a)$ is the variance for point $x \in I$ if we quantize $x$ to an endpoint of $I = [a, b]$.
That is, $\err (x, I)$ is the variance of the (unique) distribution $D$ supported on ${a, b}$ so that $\E_{X \sim D} [X] = x$.

Given an interval $I \subseteq [0, 1]$, we let $\setX_I$ be the set of $x_j \in \setX$ contained in $I$.
We also define $\err (\setX, I) = \sum_{x_j \in I} \err (x_j, I)$.
Given a partition $\setI$ of $[0, 1]$, we let $\err (\setX, \setI) = \sum_{I \in \setI} \err (\setX, I)$.
We let the optimum solution be $\setI^* = \argmin_{|\setI| = k} \err (\setX, \setI)$, breaking ties randomly. 

\subsection{Dynamic Programming}

We first present a dynamic programming algorithm that solves the above problem in an exact way. In the next subsection, we present a more practical approximation algorithm that only needs to scan all data points \emph{once}.

This optimization problem is non-convex and non-smooth. 
We start from the observation that there exists an optimal solution that places endpoints \emph{at input points}. 

\begin{lemma}
\label{lem:discrete}
There is a $\setI^*$ so that all endpoints of any $I \in \setI^*$ are in $\Omega \cup \{0, 1\}$.
\end{lemma}

Therefore, to solve the problem in an exact way, we just need to select a subset of data points in $\Omega$ as quantization points. Define $T(k, m)$ be the optimal total variance for points in $[0, d_m]$ with $k$ quantization levels choosing $d_m=x_m$ for all $m=1,2,\cdots, N$. Our goal is to calculate $T(s, N)$. This problem can be solved by dynamic programing using the following recursion
\[
T(k, m) = \min_{j\in \{k-1, k, \cdots, m-1\}} T(k-1,j) + V(j,m),
\]
where $V(j,m)$ denotes the total variance of points falling in the interval $[d_j, d_m]$. The complexity of calculating the matrix $V(\cdot, \cdot)$ is $O(N^2 + N)$ and the complexity of calculating the matrix $T(\cdot, \cdot)$ is $O(kN^2)$. The memory cost is $O(kN + N^2)$. 

\begin{figure}[t]
\centering    
\includegraphics[width=0.5\columnwidth]{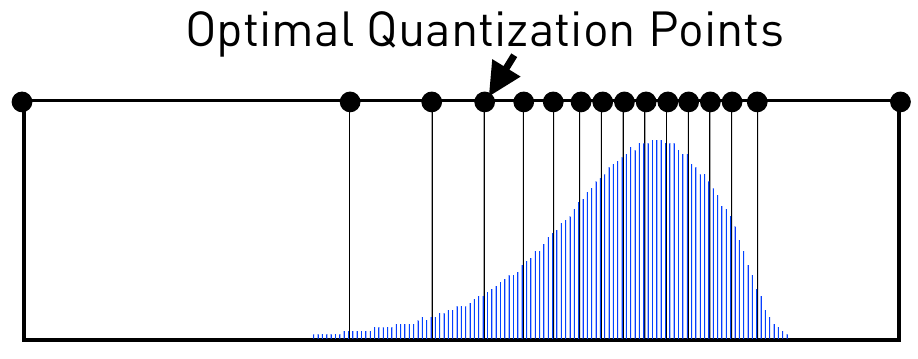} 
\caption{Optimal quantization points calculated with
dynamic programming given a data distribution. }
\label{fig:optimalquantization}
\end{figure} 

\subsection{Heuristics}

The exact algorithm has a complexity that is quadratic in the number of data points, which may be impractical. To make our algorithm practical,
we develop an approximation algorithm that only needs to scan all data points once and has linear complexity to $N$.

\paragraph*{Discretization.}

We can discretize the range $[0,1]$ into $M$ intervals, i.e., $[0,d_1), [d_1, d_2), \cdots, [d_{M-1}, 1]$ with $0< d_1<d_2<\cdots < d_{M-1}<1$. We then restrict our algorithms to only choose $k$ quantization points within these $M$ points, instead of all $N$ points in the exact algorithm. The following result bounds the quality of this approximation.

\begin{theorem} \label{thm:optQ}
Let the maximal number of data points in each ``small interval'' (defined by $\{d_m\}_{m=1}^{M-1}$) and the maximal length of small intervals be bounded by $bN/M$ and $a/M$, respectively. Let ${\mathcal{I}^*} := \{l^*_j\}_{k=1}^{k-1}$ and $\hat{\mathcal{I}}^* :=\{\hat{l}^*_k\}_{k=1}^{k-1}$ be the optimal quantization to \eqref{eq:opt_Q} and the solution with discretization. Let $cM/k$ be the upper bound of the number of small intervals crossed by any ``large interval'' (defined by ${\mathcal{I}}^*$). Then we have the discretization error bounded by
\[
 \mathcal{MV}(\hat{\mathcal{I}}^*) -  \mathcal{MV}({\mathcal{I}}^*) \leq {a^2b k \over 4 M^3} + {a^2bc^2 \over Mk}.
\]
\end{theorem}

Theorem~\ref{thm:optQ} suggests that the mean variance using the discrete variance-optimal quantization will converge to the optimal with the rate $O(1/Mk)$.

\paragraph*{Dynamic Programming with Candidate Points.}
Notice that we can apply the same dynamic programming approach given $M$ candidate points. 
In this case, the total computational complexity becomes $O((k+1)M^2 + N)$, with memory cost 
$O(kM + M^2)$. Also, to find the optimal quantization,  we only need to scan all $N$ numbers once.
Figure~\ref{fig:optimalquantization} illustrates an example output for our algorithm.

\paragraph*{$2$-Approximation in Almost-Linear Time.} 
In the supplementary material, we present an algorithm which, given $\Omega$ and $k$, provides a split using at most $4 k$ intervals, which guarantees a $2$-approximation of the optimal variance for $k$ intervals, using $O( N \log N )$ time. This 
is a new variant of the algorithm by~\cite{acharya2015fast} for the histogram recovery problem. 
We can use the $4k$ intervals given by this algorithm as candidates for the DP solution, to get a general $2$-approximation using $k$ intervals in time $O( N \log N + k^3)$. 

\subsection{Applications to Deep Learning}

In this section, we show that it is possible 
to apply optimal quantization to
training deep neural networks.

\paragraph*{State-of-the-art.} We focus on
training deep neural networks with a quantized
model. Let $\mathcal{W}$ be the model and 
$l(\mathcal{W})$ be the loss function. State-of-the-art quantized networks,
such as XNOR-Net and QNN, replace $\mathcal{W}$
with the quantized version $Q(\mathcal{W})$, and optimize
for
\[
\min_{\mathcal{W}} l(Q(\mathcal{W})).
\]
With a properly defined 
$\frac{\partial Q}{\partial{\mathcal{W}}}$, we can
apply the standard backprop 
algorithm.
Choosing the quantization function $Q$ is
an important design decision. For 1-bit quantization,
XNOR-Net searches the optimal quantization point. However, for multiple bits,
XNOR-Net, as well as other approaches such as QNN, resort
to uniform quantization.

\paragraph*{Optimal Model Quantization for Deep Learning.}

We can apply our optimal quantization strategy 
and use it as the quantization function $Q$
in XNOR-Net. Empirically, this results in 
quality improvement
over the default {\em multi-bits} quantizer in XNOR-Net. 
In spirit, our approach is similar to the 1-bit quantizer of
XNOR-Net, which is equivalent to our approach when the data
distribution is symmetric---we extend this
to multiple bits in a principled way. Another related work
is the uniform quantization strategy 
in {\em log domain}~\cite{miyashita2016convolutional},
which is similar to our approach when the data distribution
is ``log uniform.'' However, our approach does not rely on
any specific assumption of the data distribution.
\citet{Han:2016:ICLR} use $k$-means to
compress the model for {\em inference}~---$k$-means
optimizes for a similar, but different, objective
function than ours. In this paper, we 
develop a dynamic
programming algorithm to do optimal stochastic quantization efficiently.

\section{Non-Linear Models}

In this section, we extend our framework to approximate arbitrary classification losses within arbitrarily small bias. 

\subsection{Quantizing Polynomials} 

Given a degree $d$ polynomial $P(x) = \sum_{i = 0}^{d} m_i z^i$,
our goal is to evaluate at $\vec{a}^\top \vec{x}$, while quantizing $\vec{a}$, so as to preserve the value of $P( \vec{a}^\top \vec{x})$ in expectation. 

We will use $d$ independent quantizations of $\vec{a}$, $Q_1(\vec{a}), Q_2(\vec{a}), \ldots, Q_d(\vec{a})$. 
Given these quantizations, our reconstruction of the polynomial at $( \vec{a}^\top \vec{x})$ will be 
$$ Q(P) := \sum_{i = 0}^d m_i \prod_{j \leq i} Q_j(\vec{a})^\top \vec{x}.$$

The fact that this is an unbiased estimator of $P( \vec{a}^\top \vec{x} )$ follows from the independence of the quantizations. Using Lemma~\ref{lem:quant-facts} yields:

\begin{lemma}
\label{lem:poly-sec-moment-bound}
	$\E[ Q(P)^2 ] \leq \left(\sum_{i = 0}^d m_i r(s)^i (\vec{a}^\top \vec{x})^i\right)^2.$
\end{lemma}

\subsection{Quantizing Smooth Classification Losses}

We now examine a standard classification setting, where samples $[(\vec{a}_i, b_i)]_i$ are drawn from a distribution $\mathcal{D}$. Given a smooth loss function $\ell: \R \rightarrow \R$, we wish to find $\vec{x}$ which minimizes $\E_{\mathcal{D}} [ \ell( b \cdot \vec{a}^\top \vec{x}) ]$. The gradient of $\ell$ is given by 
$$ \nabla_\vec{x} (b \cdot \vec{a}^\top \vec{x}) = b \ell' (b \cdot \vec{a}^\top \vec{x}) \vec{a}.$$

Assume normalized samples, i.e. $\| \vec{a}_i \|_2 \leq 1, \forall i$, and that $\vec{x}$ is constrained such that $\| \vec{x} \|_2 \leq R$, for some real value $R > 0$. We wish to approximate the gradient within some target accuracy $\epsilon$. 

To achieve this, fix a minimal-degree polynomial $P$ such that $|P(z) - \ell'(z)| \leq \epsilon, \forall z \leq R$. Assume this polynomial is known to both transmitter (sample source) and receiver (computing device). The protocol is as follows. 
\begin{itemize}
	\item For a given sample $(\vec{a}_i, b_i)$ to be quantized, the source will transmit $b_i$, as well as $d + 1$ independent quantizations $Q_1, Q_2, \ldots, Q_{d + 1}$ of $\vec{a}_i$. 
	\item The receiver computes $b \cdot Q(P) Q_{d + 1} ( \vec{a}_i )$ and uses it as the gradient.
\end{itemize}

It is easy to see that the bias in each step is bounded by $\epsilon$. 
We can extend Lemma~\ref{lem:poly-sec-moment-bound} to obtain a general guarantee on convergence. 

\begin{lemma} \label{lem:OPT}
	For any $\epsilon > 0$ and any convex classification loss function $\ell: \R \rightarrow \R$, there exists a polynomial degree $D(\epsilon, \ell)$ such that the polynomial approximation framework converges to within $\epsilon$ of OPT.  
\end{lemma}

\paragraph{Chebyshev Approximations.} 
For \emph{logistic loss}, with sigmoid gradient, we notice that polynomial approximations have been well studied. In particular, we use the Chebyshev polynomial approximation of~\cite{vlcek2012chebyshev}. 

\subsection{Quantizing Non-Smooth Classification Losses}

Our techniques further extend to convex loss functions with non-smooth gradients.  
For simplicity, in the following we focus on SVM, whose gradient (the step function), is discontinuous. 
This gradient is hard to approximate generally by polynomials; yet, the problem is approachable on intervals of the type $[-R, R] \setminus [-\delta, \delta]$, for some small parameter $\delta > 0$~\cite{frostig2016principal, allen2016faster}; the latter reference provides the optimal approximation via Chebyshev polynomials, which we use in our experiments. 

The key challenge is that these results do not provide any non-trivial guarantees for our setting, since gradients within the interval $[-\delta, \delta]$ can differ from the true gradient by $\Omega (1)$ in expectation. In particular, due to quantization, the gradient might be \emph{flipped}: 
its relative value with respect to $0$ changes, which corresponds to having the \emph{wrong} label for the current sample.\footnote{Training SVM with noisy labels has been previously considered, e.g.~\cite{Natarajan:2013:NIPS}, but in a setting where labels are corrupted uniformly at random. It is not hard to see that label corruptions are not uniform random in this case.}
We show two approaches for controlling the error resulting from these errors.

The first is to just ignore such errors: under generative assumptions on the data, we can prove that quantization does not induce significant error. 
In particular, the error vanishes by taking more data points.
The second approach is more general: we use ideas from dimensionality reduction, specifically, low randomness Johnson-Lindenstrauss projections, to detect (with high probability) if our gradient could be flipped. If so, we refetch the full data points. 
This approach is always correct; however, it requires more communication.
Under the same generative assumptions, we show that the additional communication is \emph{sublinear} in the dimension.
Details are in the supplementary material.

\paragraph{Practical Considerations.} The above strategy introduces a precision-variance trade-off, since increasing the precision of approximation (higher polynomial degree) also increases the variance of the gradient. 
Fortunately, we can reduce the variance and increase the approximation quality by increasing the density of the quantization. 
In practice, a total of $8$ bits per sample is sufficient to ensure convergence for both hinge and logistic loss. 

\paragraph*{The Refetching Heuristic.}
The second theoretical approach inspires the following heuristic. 
Consider hinge loss, i.e.  $\sum_{k=1}^K \max(0, 1 - b_k \a_k^\top \x)$. 
We first transmit a single low-precision version of $\a_k$, and   
calculate upper and lower bounds on $b_k \a_k^\top \x$ at the receiver.
If the sign of $1-b_k \a_k^\top \x$ cannot change because of quantization, then we apply the approximate gradient. 
If the sign could change, then we {\em refetch} the data at full precision.
In practice, this works for 
8-bit while only refetching $<5\%$ of the data.

\begin{table}[t]
\small
\centering
\begin{tabular}{crrrr}
\hline
\multicolumn{4}{c}{\bf Regression}\\
Dataset           & Training Set & Testing Set & \# Features  \\
\hline
Synthetic 10   & 10,000        & 10,000       & 10               \\
Synthetic 100  & 10,000        & 10,000       & 100              \\
Synthetic 1000 & 10,000        & 10,000       & 1,000           \\
YearPrediction & 463,715       & 51,630       & 90                  \\
cadata         & 10,000        & 10,640       & 8                   \\
cpusmall       & 6,000         & 2,192        & 12     \\
\hline
\hline
\multicolumn{4}{c}{\bf Classification}\\
Dataset           & Training Set & Testing Set & \# Features \\
\hline
cod-rna        & 59,535        & 271,617      & 8    \\
gisette        & 6,000         & 1,000        & 5,000  \\  
\hline
\hline
\multicolumn{4}{c}{\bf Deep Learning}\\
Dataset           & Training Set & Testing Set & \# Features \\
\hline
CIFAR-10        & 50,000        & 10,000      &$32\times 32\times 3$     \\
\hline
\hline
\multicolumn{4}{c}{\bf Tomographic Reconstruction}\\
Dataset           & \# Projections & Volumn Size & Proj. Size \\
\hline
                  & $128$            & $128^3$      & $128^3$     \\
\hline
\end{tabular}
\caption{Dataset statistics.}
\label{table:dataset}
\end{table}

\section{Experiments} \label{sec:exp}

We now provide an empirical validation of
our ZipML framework.

\paragraph{Experimental Setup.} 
Table~\ref{table:dataset} shows the 
datasets we use. 
Unless otherwise noted, we always
use diminishing stepsizes $\alpha/k$,
where $k$ is the current number of
epoch. We tune 
$\alpha$ for the full precision
implementation, and use the
same initial step size for 
our low-precision 
implementation. (Theory and
experiments imply that the low-precision
implementation often favors smaller step size. 
Thus we do not tune step sizes for the low-precision 
implementation, as this can only improve the accuracy of our approach.) 

\paragraph*{Summary of Experiments.}
Due to space limitations, we only report on {\bf Synthetic 100} for regression, and on 
{\bf gisette} for classification. 
The full version of this paper~\cite{zhang2016zipml} contains (1) several other datasets, 
and discusses (2) different
factors such as impact of the number of features, 
and (3) refetching heuristics. The FPGA implementation and design
decisions can be found in~\cite{kara2017fpga}.

\begin{figure}[t]
\centering
\includegraphics[width=0.8\columnwidth]{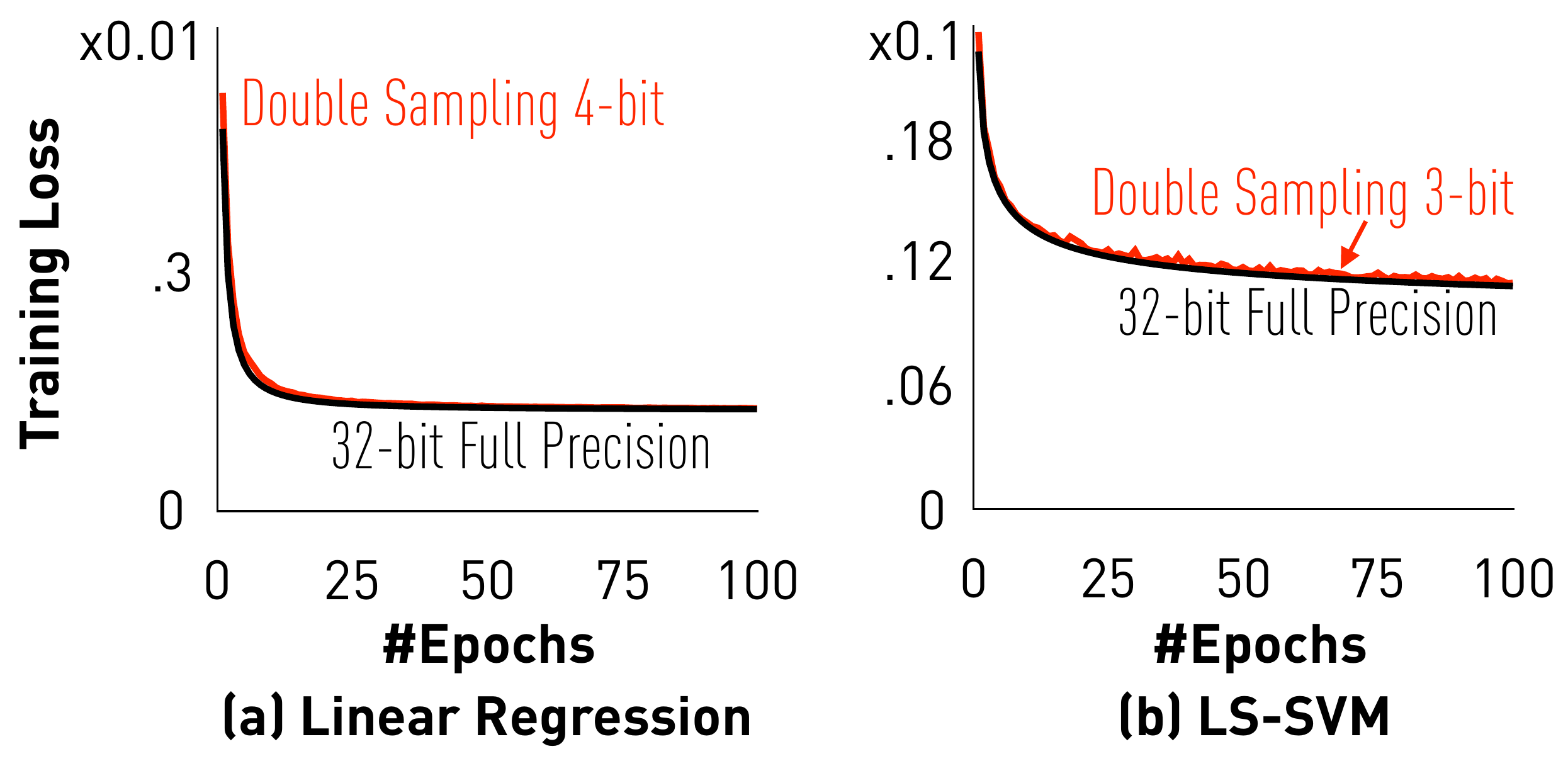} 
\caption{Linear models with end-to-end low precision.}
\label{fig:convergence}
\end{figure}

\subsection{Convergence on Linear Models}

We validate that (1) 
with double sampling, SGD with low
precision converges---in
comparable empirical 
convergence rates---to the same solution
as SGD with full precision; and
(2) implemented on FPGA, our low-precision
prototype achieves significant speedup
because of the decrease in bandwidth
consumption.

\paragraph{Convergence.}

Figure~\ref{fig:convergence} illustrates
the result of training linear models:
(a) linear
regression and (b) least squares SVMs,
with end-to-end low-precision and 
full precision. For
low precision, we pick the 
smallest number of bits that
results in a smooth convergence
curve. We compare the final 
training loss in both settings 
and the convergence rate.

We see that, for both linear regression 
and least squares SVM,
using 5- or 6-bit is always enough
to converge to the same solution
with comparable convergence rate. 
This validates our prediction that
double-sampling provides an
unbiased estimator of the gradient.
Considering the size of input
samples that we need to read, we
could potentially save 6--8$\times$ 
memory bandwidth compared to using 
32-bit. 

\paragraph{Speedup.}
We implemented our low-precision 
framework on a state-of-the-art 
FPGA platform. The detailed 
implementation is described in ~\cite{kara2017fpga}.
This implementation assumes the input
data is already quantized and
stored in memory (data can be
quantized during the
first epoch).

Figure~\ref{fig:speedup} illustrates 
the result of (1) our FPGA
implementation with quantized data,
(2) FPGA implementation with 32-bit
data, and (3) Hogwild! running with
10 CPU cores. 
Observe that all approaches
converge to the same solution.
FPGA with quantized data converges
6-7$\times$ faster
than FPGA with full precision
or Hogwild!. The FPGA implementation
with full precision is
memory-bandwidth bound, and by using our framework on quantized data, we save 
up to 8$\times$ memory-bandwidth, which
explains the speedup.

\begin{figure}[t]
\centering
\includegraphics[width=0.8\columnwidth]{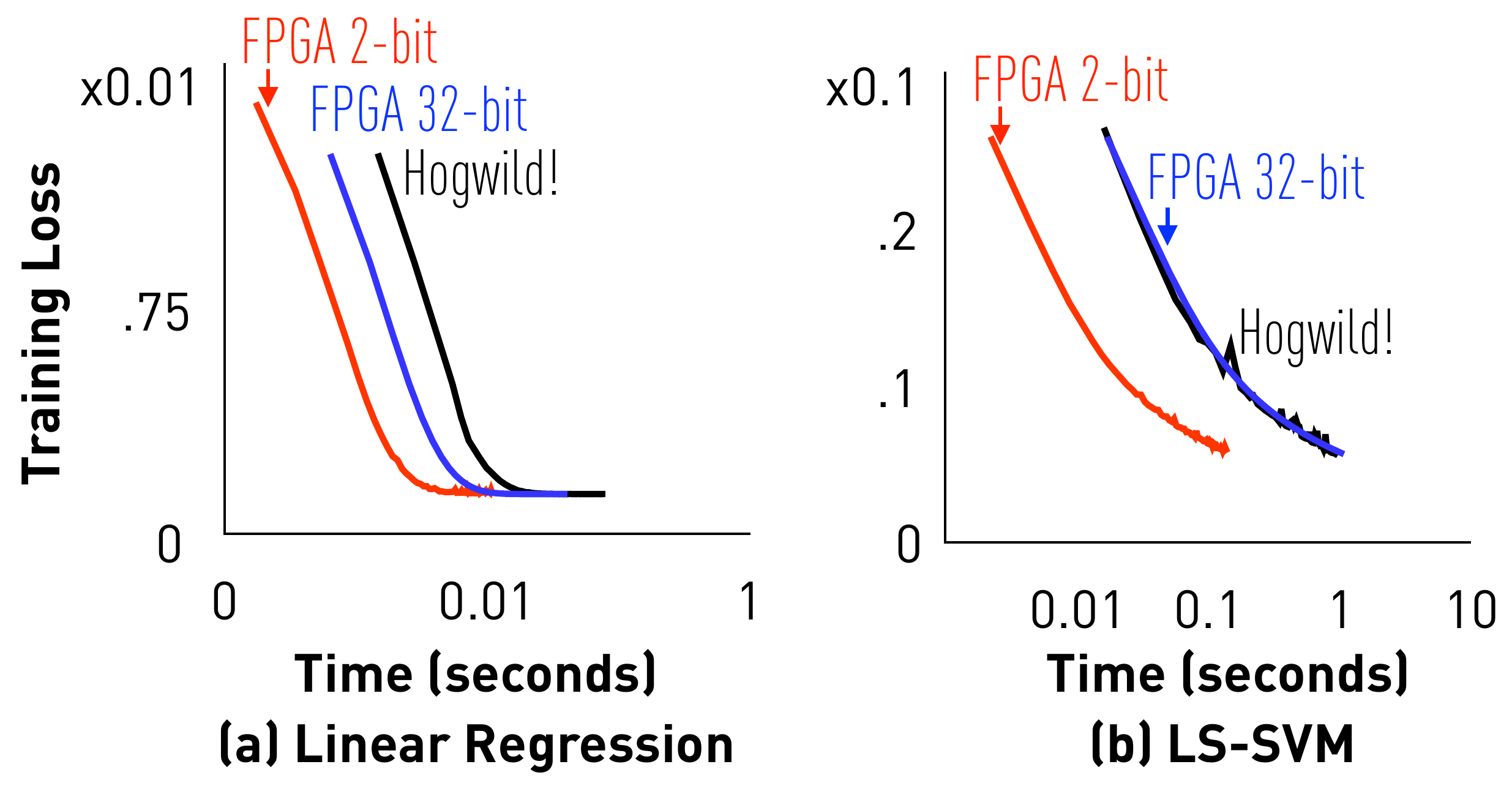} 
\caption{FPGA implementation of linear models.}
\label{fig:speedup}
\end{figure}

\begin{figure}[t]
\centering
\includegraphics[width=0.8\columnwidth]{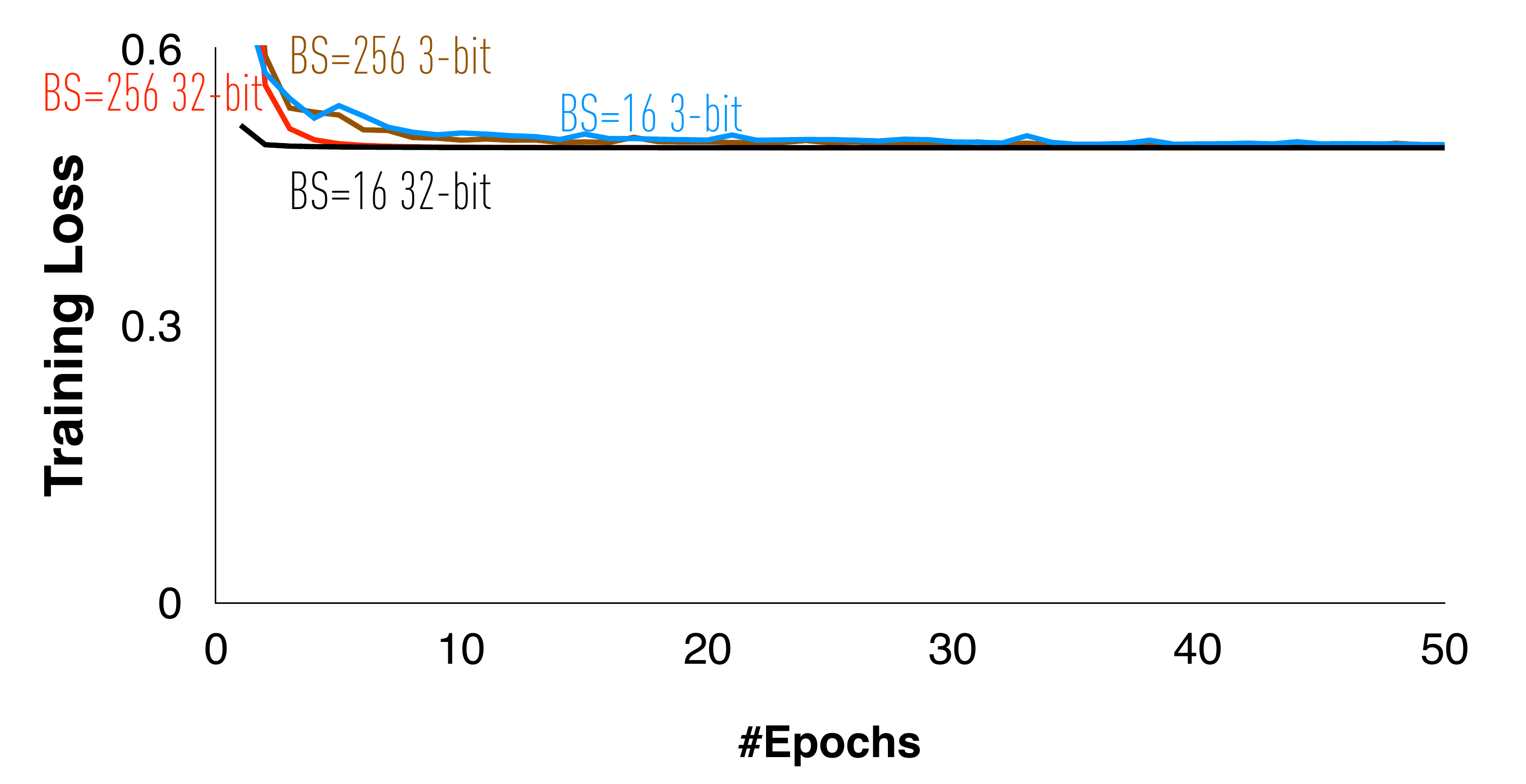} 
\caption{Impact of Using Mini-Batch. BS=Batch Size.}
\label{fig:minibatch}
\end{figure}

\paragraph{Impact of Mini-Batching.}

We now validate the``sensitivity'' of the algorithm to the precision under batching. Equation~\ref{eqn:varbound} suggests that, as we increase batch size, the variance term corresponding to input quantization may start to dominate the variance of the stochastic gradient. However, in practice and for reasonable parameter settings, we found this does not occur: convergence trends for small batch size, e.g. 1, are the same as for larger sizes, e.g. 256.
Figure~\ref{fig:minibatch} shows that, if we use larger mini-batch size (256), we need more epochs
than using smaller mini-batch size (16) to converge, but for the quantized version, actually the one with larger mini-batch size converges faster.

\begin{figure}[t]
\centering
\includegraphics[width=0.8\columnwidth]{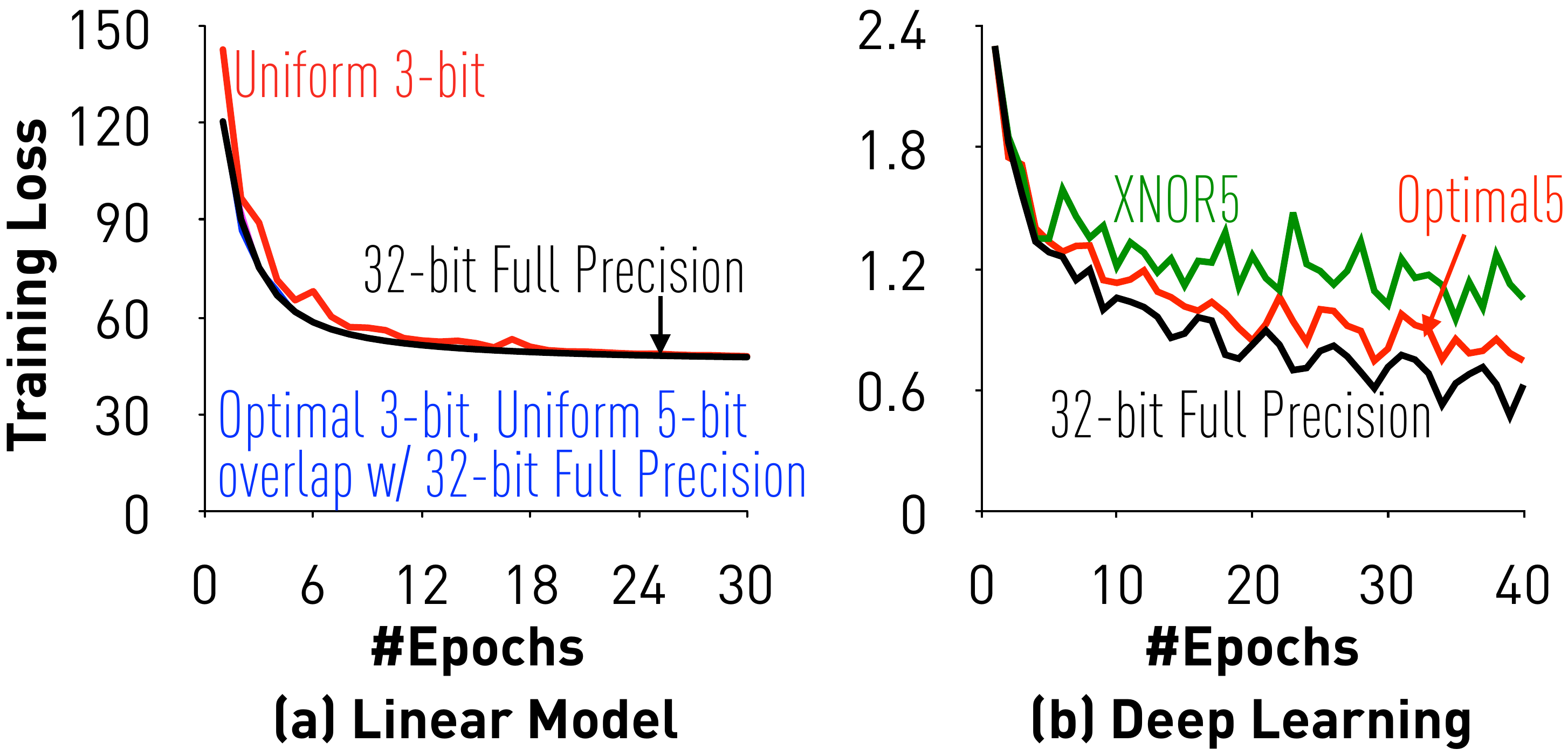} 
\caption{Optimal quantization strategy.}
\label{fig:optimal}
\end{figure}

\subsection{Data-Optimal Quantization Strategy}

We validate that, with our data-optimal quantization strategy, we can 
significantly decrease the number of 
bits that double-sampling requires to 
achieve the same convergence.
Figure~\ref{fig:optimal}(a) illustrates
the result of using 3-bit and 5-bit
for uniform quantization and optimal 
quantization on the {\bf YearPrediction}
dataset. 
Here, we only consider quantization on data, but not on gradient or model, because to compute the
data-optimal quantization, we need to have access to all data and assume the data doesn't change too much, which is not the case for gradient or model.
The quantization points are calculated for each feature for both uniform quantization and optimal quantization.
We see that,
while uniform quantization needs 5-bit
to converge smoothly, optimal
quantization only needs 3-bit. 
We save almost $1.7\times$ number of 
bits by just allocating quantization points carefully.

\paragraph{Comparision with uniform quantization.}

\begin{figure}[t]
\centering
    \includegraphics[width=1\columnwidth]{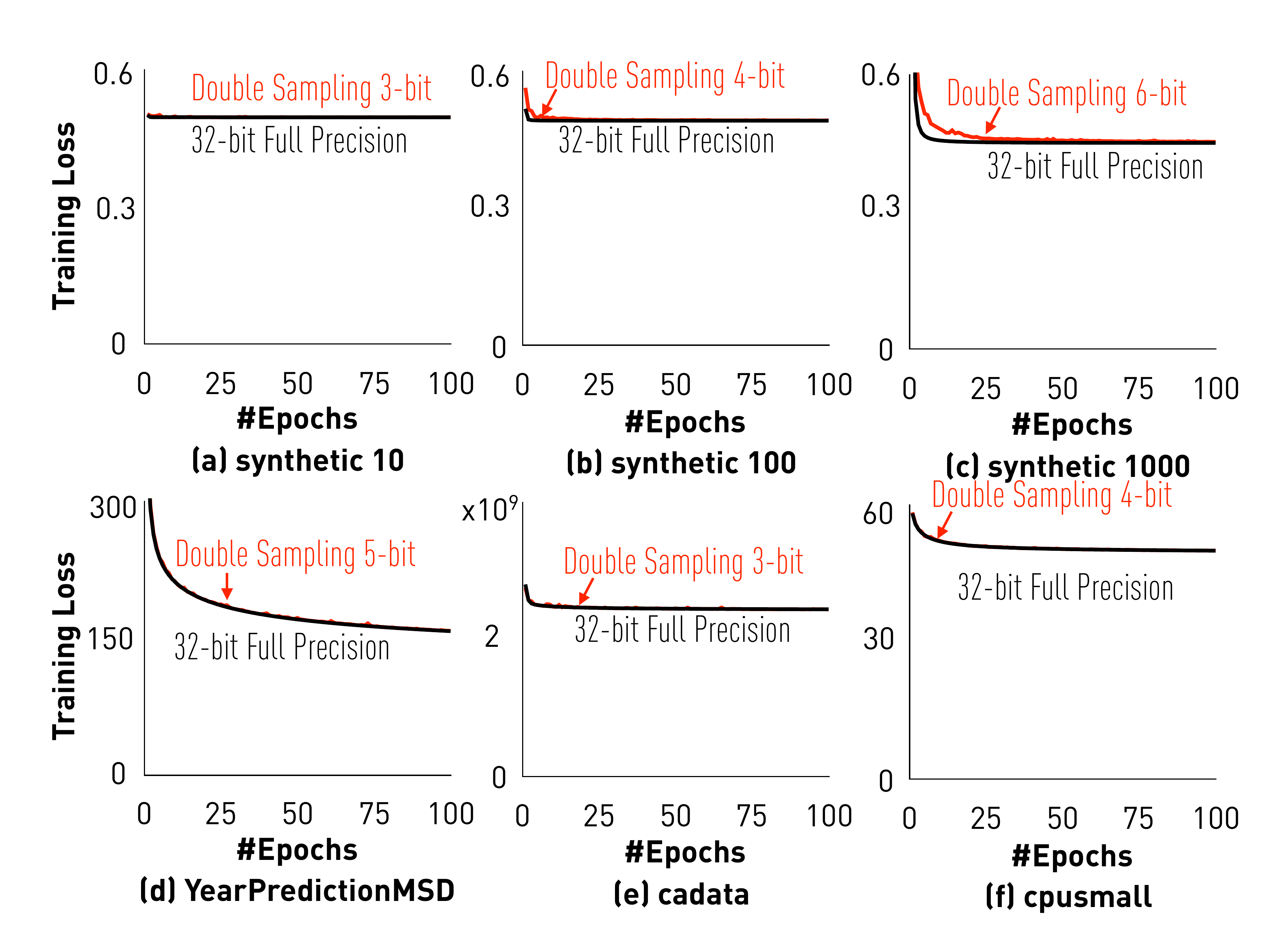} 
\caption{Linear regression with \emph{quantized data} on multiple datasets.}
\label{fig:lr-qd}
\end{figure}

We validate that, with our data-optimal quantization strategy, we can 
significantly increase the convergence speed.

Figure~\ref{fig:lr-qd} illustrates
the result of training linear regression models: 
with uniform quantization points and 
optimal quantization points. Here, notice that we only quantize data,
but not gradient or model.
We see that, if we use same number of bits, optimal quantization
always converges faster than uniform quantization and the loss curve
is more stable, because the variance induced by quantization is smaller. 
As a result, with our data-optimal quantization strategy,
we can either (1) get up to $4\times$ faster convergence speed
with the same number of bits;
or (2) save up to $1.7\times$ bits while getting the same convergence speed.

We also see from Figure~\ref{fig:lr-qd} (a) to (c) that if the dataset has more
features, usually we need more bits for quantization,
because the variance induced by quantization
is higher when the dimensionality is higher.
\subsection{Extensions to Deep Learning}

We validate that our data-optimal quantization
strategy can be used in training deep neural
networks. We take Caffe's CIFAR-10 tutorial~\cite{Caffe:CIFAR10}
and compare three different quantization
strategies: (1) Full Precision, (2) XNOR5, 
a XNOR-Net implementation that, following
the multi-bits strategy in
the original paper, quantizes data into
five uniform levels, and (3)
Optimal5, our quantization strategy with
five optimal quantization levels. As
shown in Figure~\ref{fig:optimal}(b), Optimal5
converges to a significantly lower training 
loss compared with XNOR5. Also,
Optimal5 achieves $>$5 points higher testing accuracy over XNOR5.
This illustrates the improvement
obtainable by training a neural network with
a carefully chosen quantization strategy.

\begin{figure}[t]
\centering
\includegraphics[width=0.8\columnwidth]{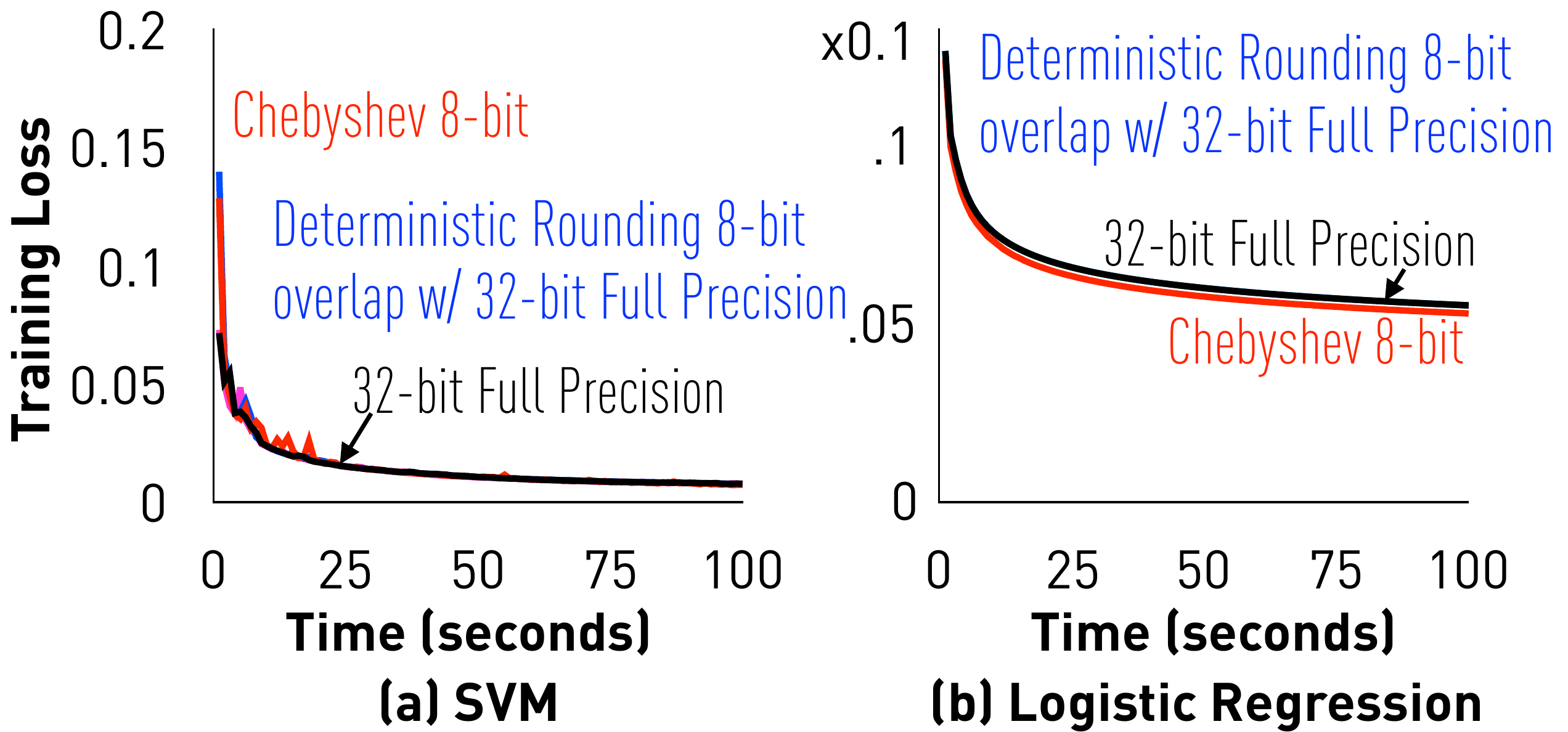} 
\caption{Non-linear models with Chebyshev approximation.}
\label{fig:chebyshev}
\end{figure}

\subsection{Non-Linear Models}

We validate that (1) our Chebyshev 
approximation approach is able to
converge to almost the same solution 
with 8-bit precision for both SVM
and logistic regression;
and (2) we are nevertheless able to construct
a straw man with 8-bit deterministic 
rounding or naive stochastic rounding
to achieve the same quality and convergence 
rate.

\paragraph{Chebyshev Approximations.}

Figure~\ref{fig:chebyshev} illustrates
the result of training SVM
and logistic regression 
with Chebyshev approximation. Here,
we use Chebyshev polynomials up to
degree 15 (which requires 16 samples
that can be encoded with 4 extra 
bits). For each sample, the precision
is 4-bit, and therefore, in total
we use 8-bit for each single number
in input samples. We see that, 
with our quantization framework,
SGD converges to similar training loss 
with a comparable empirical convergence 
rate for both SVM and logistic regression.
We also experience no loss in test accuracy.

\paragraph{Negative Results.}

We are able to construct the following,
much simpler strategy that also
uses 8-bit to achieve the same quality
and convergence rate as our
Chebyshev. In practice, as both
strategies incur bias on the result,
we do {\em not} see strong reasons to
use our Chebyshev approximation, thus
we view this as a negative result.
As shown in Figure~\ref{fig:chebyshev},
if we simply round the input samples
to the nearest 8-bit fix point
representation (or do rounding
stochastically), we achieve the
same, and sometimes better,
convergence than our Chebyshev 
approximation.

\section{Related Work} 

There has been significant work on ``low-precision SGD''~\cite{DeSa:NIPS:2015,Alistarh:2016:ArXiv}. 
These results provide
theoretical guarantees only for quantized gradients.
The model and input samples, on the other hand, are much more difficult
to analyze because of the non-linearity. We focus on {\em end-to-end}
quantization, for all components.

\paragraph{Low-Precision Deep Learning.}

Low-precision training of deep neural networks has been studied
intensively and many heuristics work well for a subset of networks.
OneBit SGD~\cite{Frank:2014:Interspeech} provides
a gradient compression heuristic developed in the context of deep 
neural networks for speech recognition. There are successful 
applications of end-to-end quantization to training neural networks that 
result in little to no quality loss~\cite{hubara2016quantized,
rastegari2016xnor,zhou2016dorefa,miyashita2016convolutional,li2016ternary,gupta2015deep}. They quantize weights, activations, and gradients 
to low precision (e.g., 1-bit) and revise the backpropagation 
algorithm to be aware of the quantization function.
The empirical success of this work inspired this paper, in which we try
to provide a {\em theoretical} understanding of end-to-end low-precision
training for machine learning models.
Another line of research concerns inference and model
compression of a pre-trained model~\cite{vanhoucke2011improving,gong2014compressing,Han:2016:ICLR,lin2016fixed,kim2016bitwise,kim2015compression,wu2016quantized}.
In this paper, we focus on training and leave the study of
inference for future work.

\paragraph{Low-Precision Linear Models.}

Quantization is a fundamental topic studied by the
DSP community, and there has been research on
linear regression models in the presence of quantization
error or other types of noise. For example,
\citet{Gopi:2013:ICML} studied compressive sensing
with binary quantized measurement, and a two-stage algorithm was proposed to recover the sparse high-precision solution up to a scale factor.
Also, the
classic errors-in-variable model~\cite{Hall:2008:Book}
could also be relevant if quantization is treated 
as a source of ``error.'' In this paper, we scope
ourselves to the context of stochastic gradient descent, 
and our insights go beyond simple linear models.
For SVM the straw man approach 
can also be seen as a very simple case of kernel 
approximation~\cite{Cortes:2010:AISTATS}.

\paragraph{Other Related Work.} Precision of data
representation is a key design decision
for configurable hardwares such as FPGA. There have been attempts to
lower the precision when training on such hardware~\cite{Kim:2011:ICASSP}. 
These results are mostly empirical; we
aim at providing a theoretical understanding, which 
enables new algorithms.

\section{Discussion}
\label{sec:conclusions}

Our motivating question was whether end-to-end low-precision data representation can enable efficient computation with convergence guarantees. 
We show that a relatively simple stochastic quantization framework can achieve this for linear models. 
With this setting, as little as two bits per model dimension are sufficient for good accuracy, and can enable a fast FPGA implementation.  

For non-linear models, the picture is more nuanced. 
In particular, we find that our framework can be generalized to this setting, and that in practice \emph{8-bit is sufficient} to achieve good accuracy on a variety of tasks, such as SVM and logistic regression. 
However, in this generalized setting, naive rounding has similar performance on many practical tasks. 

It is interesting to consider the rationale behind these results. Our framework is based on the idea of \emph{unbiased approximation} of the original SGD process. For linear models, this is easy to achieve. For non-linear models, this is harder, and we focus on guaranteeing arbitrarily low bias. 
However, for a variety of interesting functions such as hinge loss, guaranteeing low bias requires complex approximations. In turn, these increase the variance. The complexity of the approximation and the resulting variance increase force us to increase the \emph{density} of the quantization, in order to achieve good approximation guarantees. 

\newpage~\newpage~\newpage~
\appendix
\begin{center}
\textbf{\large Supplemental Materials: Training Models with End-to-End Low Precision:\\
The Cans, the Cannots, and a Little Bit of Deep Learning}
\end{center}

This supplemental material is the laboratory of this project. All omitted proofs, additional theorems, and experiment details can be found from corresponding sections.

\section{Preliminaries}

\subsection{Computational Model}

We consider a computational model illustrated in
Figure~\ref{fig:model}.  
In this context, SGD is often bounded by the bandwidth
of data movements cross these components. 
In particular, we consider the convergence properties of the algorithm when a lossy compression scheme is applied to the data (samples), 
gradient, and model, for the purpose of reducing the communication cost of the algorithm. 
It is interesting to consider how lossy compression impacts the update step in SGD. Let $Q( \vec{v} )$ denote the compression scheme applied to a vector $\vec{v}$. 

\begin{itemize}
    \item \textbf{Original iteration}: $$\x_{t + 1} \leftarrow \x_t - \gamma \g_k (\x_t, \vec{a}_t).$$
    \item \textbf{Compressed gradient}: $$\x_{t + 1} \leftarrow \x_t - \gamma Q( \g_k (\x_t, \vec{a}_t) ).$$
    \item \textbf{Compressed model}: $$\x_{t + 1} \leftarrow \x_t - \gamma \g_k (Q(\x_t), \vec{a}_t).$$
    \item \textbf{Compressed sample}: $$\x_{t + 1} \leftarrow \x_t - \gamma \g_k (\x_t, Q(\vec{a}_t)).$$
    \item \textbf{End-to-end compression}: $$\x_{t + 1} \leftarrow \x_t - \gamma Q(\g_k (Q(\x_t), Q(\vec{a}_t))).$$
\end{itemize}

%
%

\subsection{Guarantees for SGD}
In this paper we consider SGD, a general family of stochastic first order methods for finding the minima of convex (and non-convex) functions.
Due to its generality and usefulness, there is a vast literature on SGD in a variety of settings, with different guarantees in all of these settings.
Our techniques apply fairly generally in a black box fashion to many of these settings, and so for simplicity we will restrict our attention to a fairly basic setting.
For a more comprehensive treatment, see \cite{bubeck2015convex}.

Throughout the paper, we will assume the following setting in our theoretical analysis.
Let $\mathcal{X} \subseteq \R^n$ be a known convex set, and let $f: \mathcal{X} \to \R$ be differentiable, convex, and unknown.
We will assume the following, standard smoothness condition on $f$:
\begin{definition}[Smoothness]
Let $f: \R^n \to \R$ be differentiable and convex.
We say that it is $L$-smooth if for all $\x, \y \in \R^n$, we have
\[0 \leq f(\x) - f(\y) - \nabla f(\y)^T (\x - \y) \leq \frac{L}{2} \| \x - \y \|_2^2 \; .\]
\end{definition}

We assume repeated access to stochastic gradients, which on (possibly random) input $\vec{x}$, outputs a direction which is in expectation the correct direction to move in.
Formally:
\begin{definition}
Fix $f: \mathcal{X} \to \R$.
A \emph{stochastic gradient} for $f$ with bias bound $\beta$ is a random function $g (\x)$ so that $\E [g (\x) ] = G( \x)$, where $\| G(\x) - \nabla f(\x) \|_2 \leq \beta$ for all $x \in \mathcal{X}$.
We say the stochastic gradient has second moment at most $B$ if $\E [\| g \|_2^2] \leq B$ for all $\x \in \mathcal{X}$.
We say it has variance at most $\sigma^2$ if $\E [\| g (\x) - \nabla f(\x) \|_2^2] \leq \sigma^2$ for all $\x \in \mathcal{X}$. 
\end{definition}

For simplicity, if $\beta = 0$ we will simply refer to such a random function as a stochastic gradient.
Under these conditions, the following convergence rate for SGD is well-known:

\begin{theorem}[e.g. \cite{bubeck2015convex}, Theorem 6.3]
\label{thm:sgd}
Let $\mathcal{X} \subseteq \R^n$ be convex, and let $f: \mathcal{X} \to \R$ be an unknown, convex, and $L$-smooth.
Let $\x_0 \in \mathcal{X}$ be given, and let $R^2 = \sup_{\x \in \mathcal{X}} \| \x - \x_0 \|_2^2$.
Suppose we run projected SGD on $f$ with access to independent stochastic gradients with bias bound $\beta$ and variance bound $\sigma^2$ for $T$ steps, with step size $\eta_t = 1 / ( L + \gamma^{-1})$, where $\gamma = \frac{R}{\sigma} \sqrt{\frac{2}{T}}$, and
\begin{equation}
\label{eq:sgd-conv-2}
T = O \left( R^2 \cdot \max \left( \frac{2 \sigma^2}{\epsilon^2} , \frac{L}{\epsilon} \right) \right) \; .
\end{equation}
Then $\E \left[ f \left( \frac{1}{T} \sum_{t = 0}^T \x_t \right) \right] - \min_{\x \in \mathcal{X}} f(\x) \leq \epsilon + R \beta + \frac{\eta}{2} \beta^2$.
\end{theorem} 

In particular, note that the complexity the SGD method is mainly controlled by the variance bound $\sigma^2$ we may obtain. If $\sigma = 0$, the complexity is consistent with the stochastic gradient.

\subsection{Randomized Quantization}

In this section, we give a procedure to quantize a vector or real values randomly, reducing its information content. We will denote this quantization function by $Q(\vec{v},s)$, where $s\geq 1$ is the tuning parameter. 
Let $M(\vec{v}): \R^n \rightarrow \R^n$ be a positive scaling function such that, for $\vec{v}\in \R^n$, $\frac{\vec{v}_i}{M_i(\vec{v})} \in [-1, 1]$, where $M_i(\vec{v})$ denotes the $i$th element of $M(\vec{v})$.
For $\vec{v} \neq \vec{0}$ we define

\begin{equation}
Q_i(\vec{v},s) = M_i(\vec{v}) \cdot \sgn{\vec{v}_i} \cdot \mu_i(\vec{v},s) \; , \label{equ:quant2}
\end{equation}
where $\mu_i(\vec{v},s)$'s are independent random variables defined as follows. 
Let $0 \leq \ell < s$ be an integer such that $|\vec{v}_i|/M_i(\vec{v}) \in [ \ell / s, (\ell + 1) / s ]$, that is, $\ell = \lfloor s |\vec{v}_i|/\| \vec{v} \| \rfloor$. 
Here, $p(x,s) = x s - \ell$ for any $x \in [0,1]$.
Then 
\[
\mu_i(\vec{v},s) = \left\{ \begin{array}{ll}
         \ell / s & \mbox{with probability $1 - p\left(\frac{|\vec{v}_i|}{M(\vec{v})},s\right)$};\\
         (\ell + 1) / s & \mbox{otherwise}. \end{array} \right.
\]
If $\vec{v} = \vec{0}$, then we define $Q(\vec{v},s) = \vec{0}$.
For any such choice of $M_i$, we have the following properties, which generalize Lemma 3.4 in~\cite{Alistarh:2016:ArXiv}.
The proofs follow immediately from those in \cite{Alistarh:2016:ArXiv}, and so we omit them for conciseness.
\begin{lemma}
\label{lem:quant-facts-2}
 For any $\vec{v} \in \R^n$, we have that 
 \begin{itemize} 
 \item (Sparsity) $\E[ \|Q(\vec{v}, s)\|_0]\leq
 s^2 +\sqrt{n}$ , 
 \item (Unbiasedness) $\E [Q (\vec{v},s)] = \vec{v}$ , and
 \item (Second Moment Bound) 
$\E [\| Q (\vec{v},s) \|_2^2] \leq r M^2$, where $M = \max_i M_i (\vec{v})$, and 
\[
r = r(s) = \left( 1 + \frac{1}{s^2} \sum_{i = 1}^n p\left( \frac{|\vec{v}_i|}{M_i },s \right) \right) \; .
\]
 \end{itemize}
\end{lemma}

We now discuss different choices of the scaling function $M_i(\vec{v})$.

\paragraph*{``Row Scaling''.}

One obvious choice that was suggested in \cite{Alistarh:2016:ArXiv} is to have $M_i(\vec{v}) = \| \vec{v} \|_2$, in this way, we
always have $\frac{\vec{v}_i}{M_i(\vec{v})} \in [-1, 1]$ and all $M_i(\vec{v})$ are the same
such that we can store them only once.
When the 
In the following, we will often use the version with $s = 1$, which is as follows. 
\begin{equation}
\label{equ:quant1}
Q_i(\vec{v}) = \| \vec{v} \|_2 \cdot \sgn{\vec{v}_i} \cdot \mu_i (\vec{v}) \; ,
\end{equation}
where $\mu_i(\vec{v})$'s are independent random variables such that $\mu_i(\vec{v}) = 1$ with probability $|\vec{v}_i| / \| \vec{v} \|_2$, and $\mu_i(\vec{v}) = 0$, otherwise. If $\vec{v} = \vec{0}$, we define $Q(\vec{v}) = \vec{0}$. 
Obviously, if
all vectors $\vec{v}$ are scaled to have unit $\ell_2$ norms, $M(\vec{v}) \equiv 1$
and therefore, we can also omit this term.
Moreover, it was shown in \cite{Alistarh:2016:ArXiv} that for this choice of $M_i$, the function $r$ can be upper bounded by
\[
r(s) \leq \rrow (s) = 1 + \min\left( \frac{n}{s^2}, \frac{\sqrt{n}}{s} \right) \; .
\]

\paragraph*{``Column Scaling''.}
Let $\vec{v} \in \R^n$ be a sample and $V \subset \R^n$ be the set of sample vectors. 
We can obtain the upper and lower bound for each feature, that is,
\[
\text{min}_i \le \vec{v}_i \le  \text{max}_i\quad \vec{v} \in V
\]
is to have $M_i(\vec{v}) = \max(|\text{min}_i|, |\text{max}_i|)$.
When the input samples are stored as a matrix in which each row corresponds
two a vector $\vec{v}$, getting $\min_i$ and $\max_i$
is just to getting
the $\min$ and $\max$ for each column (feature).
Using this scheme, all input samples can share the same
$M_i(\vec{v})$ and thus can be easily stored in cache when all
input samples are accessed sequentially (like in SGD).

\paragraph*{Choice between Row Scaling and Column Scaling.}

In this working paper, we make the following choices regarding row scaling
and column scaling and leave the more general treatment to future work.
For all input samples, we always use column scaling because it is easy
to calculate $M_i$ which does not change during training. For all gradients
and models, we use row scaling because the range of values is more dynamic.

\section{Compressing the Samples for Linear Regression}

In this section, we will describe lossy compression schemes for data samples, so that when we apply SGD to solve linear regression on these compressed data points, it still provably converges.
Throughout this section, the setting will be as follows.
We have labeled data points $(\a_1, b_1), (\a_2, b_2), \ldots, (\a_K, b_K) \in \R^n \times \R$, and our goal is to minimize the function
\[
f(\x) = \frac{1}{K} \sum_{k = 1}^K \| \a_k^\top \x + b_k \|_2^2 \; ,
\]
i.e., minimize the empirical least squares loss.
The basic (unquantized) SGD scheme which solves this problem is the following: at step $\x_k$, our gradient estimator is $\g'_k = \a_{\pi(k)} (\a_{\pi(k)}^\top \x + b_{\pi(k)})$, where $\pi(k)$ is a uniformly random integer from $1$ to $m$.
In a slight abuse of notation we let $\a_k = \a_{\pi(k)}$ for the rest of the section.
Then it is not hard to see that $\E [\g'_k] = \nabla f(\x_k)$, and so this is indeed a stochastic gradient.

The rest of the section is now devoted to devising quantization schemes for $\g'_k$ when given access only to $\a_k$ and $b_k$, namely, given access only to the data points.

\subsection{Naive Random Quantization is Biased}

As a first exercise, we look at what happens when we work with the data directly in quantized form in the context of linear regression. 
The gradient becomes
\[
\g_k := Q(\a_k, s ) Q(\a_k, s)^\top \x + Q(\a_k, s) b_k.
\]
It is not hard to see that the expected value of this is in fact: 
\[
\E[\g_k] := \a_k \a_k^\top \x + \a_k b_k + D_{s, \a} \x, 
\]
where $D_{s, \a}$ is a diagonal matrix and its $i$th diagonal element is 
\[
\E[ Q(\a_i, s)^2 ] - \a_i^2.
\]

Since $D_{s, \a}$ is non-zero, we obtain a \emph{biased} estimator of the gradient, so the iteration is unlikely to converge. 
In fact, it is easy to see that in instances where the minimizer $\x$ is large and gradients become small, we will simply diverge. 
Fortunately, however, this issue can be easily fixed. 

\subsection{Double Sampling}

\paragraph{Algorithm.}
Instead of the naive estimate, our algorithm is as follows.
We generate two independent
random quantizations $Q_1$
and $Q_2$ and revise the gradient as follows:


\[
\g_k := Q_1 (\a_k, s) (Q_2 (\a_k, s)^\top \x + b_k) \; .
\]
It is not hard to see that the above is an unbiased estimator of the true gradient.\footnote{In our implementation,
we used the average gradient $\g_k := \frac{1}{2}\left(Q_1 (\a_k, s) (Q_2 (\a_k, s)^\top \x + b_k) + 
Q_2 (\a_k, s) (Q_1 (\a_k, s)^\top \x + b_k)\right)$. This version does not impact  the upper bound in our variance analysis,
but enjoys lower variance (by a constant) both theoretically and empirically.}

\paragraph{Variance Analysis.}

\begin{lemma} 
The stochastic gradient variance using double sampling above $\E\|\g_t - \g_t^{(full)}\|^2$ can be bounded by
\begin{align*}
&\Theta\left(\mathcal{TV}(\a_t) (\mathcal{TV}(\a_t)\|\x\odot \x\| + \|\a_t^\top \x\|^2 + \|\x\odot \x\|\|\a_t\|^2)\right),
\end{align*}
where $\mathcal{TV}(\a_t) := \E\|Q(\a_t) - \a_t\|^2$ and $\odot$ denotes the element product.
\end{lemma}
\begin{proof}
Let $\a$ denote $\a_t$ for short in the followed proof.
\begin{eqnarray*}
& & \E\left\|Q_1(\a)(Q_2(\a)^\top \x + b_t)\right\|^2
\\
& \leq & 
2\E\left\| (Q_1(\a)- \a) Q_2(\a)^\top \x \right\|^2 + 2\E\left\| \a(Q_2(\a) - \a)^\top \x) \right\|^2
\\
& \leq &
2 \E_1 \|Q_1(\a)-\a\|^2\E_2(Q_2(\a)^\top \x)^2 + 2 \|\a\|^2 \E((Q_2(\a)-\a)^\top\x)^2
\\
& \leq &
2 \E_1 \|Q_1(\a)-\a\|^2\E_2(Q_2(\a)^\top \x)^2 + 2 \|\a\|^2 \E((Q_2(\a)-\a)^\top\x)^2
\\
& \leq &
2 \mathcal{TV}(\a)(2 \|\a\|^2 \E((Q_2(\a)-\a)^\top\x)^2 + 2(\a^\top \x)^2)+ 2 \|\a\|^2 \E((Q_2(\a)-\a)^\top\x)^2
\\
& \leq &
\Theta\left(\mathcal{TV}(\a) (\mathcal{TV}(\a)\|\x\odot \x\| + \|\a^\top \x\|^2 + \|\x\odot \x\|\|\a\|^2)\right),
\end{eqnarray*}
which completing the proof.
\end{proof}

Let $r = r(s) = 1 + \min (n / s^2, \sqrt{n}/ s)$ be the blow-up in the second moment promised in Lemma \ref{lem:quant-facts-2}.
Then, we have the following lemma.
\begin{lemma}
    Let $\a_k, \x, b_k$ be fixed, and suppose that $\| \a_k \|_2^2 \leq A^2, \| \x \|_2^2 \leq R^2$, and $\max_i M_i (\a_k) \leq M_a$.
    Let $\g'_k = \a_k (\a_k^\top \x + b)$ be the (unquantized) stochastic gradient update.
    Then, we have 
    \[
    E_{Q_1, Q_2} [\| \g_k \|_2^2] \leq r \cdot \left( \| \g'_k \|_2^2 \cdot \frac{M_a^2}{\| \a_k \|_2^2} + \| \a_k \|_2^2 \frac{M_a^2}{s^2} R^2 \right)\; .
    \]
\end{lemma}
\begin{proof}
We have that 
\[
\E_{Q_1, Q_2}(\|\g_k\|^2) = \E_{Q_1, Q_2} [\| Q_1 (\a_k, s) (Q_2 (\a_k, s)^\top \x + b_k) \|_2^2].
\]
Next we have
\begin{align*}
    \E_{Q_1, Q_2} [\| Q_1 (\a_k, s) (Q_2 (\a_k, s)^\top \x + b_k) \|_2^2] &= \E_{Q_2} \left[ \E_{Q_1} [ (Q_2 (\a_k, s)^\top \x + b_k)^2 Q_1 (\a_k, s)^\top Q_1 (\a_k, s)] \right] \\
    &= \E_{Q_1}[ \| Q_1( \vec{a}_k, s) \|_2^2 ] \cdot \E_{Q_2} [\| \a_k (Q_2 (\a_k, s)^\top \x + b_k)\|_2^2 ] \\
    &\leq^{\mathrm{Lemma}~\ref{lem:quant-facts-2}} r M_a^2 \cdot \E [(Q_2 (\a_k, s)^\top \x + b_k)^2 ] \\
    &= r M_a^2 \left( \E [(Q_2 (\a_k, s)^\top \x)^2] + 2 b_k \E [Q_2 (\a_k, s)^\top \x] + b_k^2 \right) \\
    &= r M_a^2 \left( \E [(Q_2 (\a_k, s)^\top \x)^2] + 2 b_k \a_k^\top \x + b_k^2 \right)
\end{align*}
Moreover, we have
\begin{align*}
    E [(Q_2 (\a_k, s)^\top \x)^2] &= \x^\top \left( \E \left[ Q_2 (\a_k, s) Q_2 (\a_k, s)^\top \right] \right) \x \\
    &= \x^\top (\a_k \a_k^\top + D) \x^\top \\
    &\leq (\a_k^\top \x)^2 + \| D \|_{\mathrm{op}} \| \x \|_2^2 \; ,
\end{align*}
where $D = \mathrm{diag}_i [ (\E[Q_2 (\a_k, s)_i^2]) - (\a_k)_i^2 ] =\mathrm{diag}_i [ \mathrm{Var} [Q_2 (\a_k, s)_i] ].$ Further, we have that $\| D \|_{\mathrm{op}}  \leq M_a^2 / s^2$.
Therefore we have that:

\begin{align*}
     \E_{Q_1, Q_2} [\| Q_1 (\a_k, s) (Q_2 (\a_k, s)^\top \x + b_k) \|_2^2]  &\leq r M_a^2 \left( (\a_k^\top \x)^2 + \frac{M_a^2}{s^2} R^2 + 2 b_k \a_k^\top \x + b_k^2  \right) \\
    &= r \left( \| \g'_k \|_2^2 \cdot \frac{M_a^2}{\| \a_k \|_2^2} + \frac{A^2 M_a^2 R^2}{s^2} \right) \,
\end{align*}
as claimed, since $\| \g'_k \|_2^2 = \| \a_k \|_2^2 (\a_k^T \x + b_k)^2$.
\end{proof}
In particular, this implies the following variance bound on our quantized updates:
\begin{corollary}
\label{cor:var-bound}
    Let $\a_k, \x, b_k, \g'_k$ be as above.
    Suppose moreover $\E [\| \g'_k - \nabla f(\x_k) \|_2^2 ] \leq \sigma^2$ and $\E [\| \g'_k \|_2^2] \leq B$.
    Then, we have
    \[
    \E \left[ \| \g_k - \nabla f(\x_k) \|_2^2 \right] \leq  \sigma^2 + \left(r \frac{M_a^2}{\| \a_k \|_2^2} - 1\right) B + \frac{r A^2 M_a^2 R^2}{s^2} \; ,
    \]
    where the expectation is taken over $\g'_k$ and the randomness of the quantization.
\end{corollary}
\begin{proof}
    Observe that $\| \g_k - \nabla f(\x_k) \|_2^2 = \| \g_k - \g'_k \|_2^2 + 2 (\g_k - \g'_k)^\top (\g'_k - \nabla f(\x_k)) + \| g'_k + \nabla f(\x_k) \|_2^2$.
    Since $\E [(\g_k - \g'_k)^\top (\g'_k - \nabla f(\x_k))] = 0$, and by assumption $\E [ \| g'_k + \nabla f(\x_k) \|_2^2] \leq \sigma^2$, it suffices to bound the expectation of the first term.
    We have
    \[
     \E \left[ \| \g_k - \nabla f(\x_k) \|_2^2 \right] \leq 2 \sigma^2 + 2\E_{\g'_k} \left[ \E_{Q_1, Q_2} [ \| \g'_k - \g_k \|_2^2 \left| \right. \g'_k ] \right] \; .
    \]
    Since $\E_{Q_1, Q_2} [\g_k | \g'_k] = \g'_k $, we have that 
    \begin{align*}
    \E_{Q_1, Q_2} [ \| \g'_k - \g_k \|_2^2 \left| \right. \g'_k ] &= \E_{Q_1, Q_2} [\| \g_k \|_2^2 | \g'_k] - \| \g'_k \|_2^2 \\
    &\leq \left(r \frac{M_a^2}{\| \a_k \|_2^2} - 1\right) \| \g'_k \|_2^2 + \frac{r A^2 M_a^2 R^2}{s^2} \; ,
    \end{align*}
    from which the corollary follows.
\end{proof}

In particular, observe that this corollary essentially suggests that the quantized stochastic gradient variance is bounded by
\[
\E \left[ \| \g_k - \nabla f(\x_k) \|_2^2 \right] \leq \sigma^2 + \Theta(n/s^2) \;
\]
in the scenario when $M_i (\vec{v}) = \| \vec{v} \|_2 $.
The first term $\sigma^2$ is due to using stochastic gradient, while the second term is caused by quantization. The value of $s$ is equal to  $\lceil(2^b - 1) / 2\rceil$. Therefore, to ensure these two terms are comparable (so as not to degrade the convergence time of quantized stochastic gradient), the number of bits needs to be greater than $\Theta(\log n / \sigma)$.   

\section{Quantizing the Model}

We now assume the setting where the processor can only work with the model in \emph{quantized} form when computing the gradients. 
However, the gradient is stored in full precision---the model is quantized only when communicated. 
The gradient computation in this case is:
\begin{equation}
\label{eqn:leastsquares-2}
\g_k := \a_k \a_k^\top Q( \x , s) + \a_k b_k.
\end{equation}
It is easy to see that this gradient is unbiased, as the quantizer commutes with the (linear) gradient. 
\[
\E[ \g_k ]  := \a_k \a_k^\top \E [ Q( \x, s ) ]  + \a_k b_k = \a_k \a_k^\top \x   + \a_k b_k = \g_k .
\]
Further, the second moment bound is only increased by the variance of the quantization. 

\begin{lemma}
    \label{lem:model-quantization}
    Let $\a_k, \x, b_k$ be fixed, and suppose that $\| \a_k \|_2^2 \leq A^2,$ and $\max_i M_i (\x) \leq M_x$.
    Let $\g'_k = \a_k (\a_k^\top \x + b_k)$ be the (unquantized) stochastic gradient update.
    Then, we have
    \[
    \E [\| \g_k \|_2^2] \leq \| \g'_k \|_2^2 + \frac{A^4 M_x^2}{s^2} \; .
    \]
\end{lemma}
\begin{proof}
We have
\begin{align*}
    \E [\| \g_k \|_2^2] &= \| \a_k \|_2^2 \E \left[\left( \a_k^\top Q( \x, s)  + b_k \right)^2 \right] \\
    &= \| \a_k \|_2^2 \left( a_k^\top \E[Q(\x, s) Q(\x, s)^\top] a_k + 2 b_k \E[Q(\x, s)^\top \a_k] + b_k^2 \right) \\
    &= \| \a_k \|_2^2 \left( a_k^\top \E[Q(\x, s) Q(\x, s)^\top] a_k + 2 b_k \x^\top \a_k + b_k^2 \right) \; .
\end{align*}
As we had previously for double sampling, we have
\begin{align*}
    \a_k^\top \left( \E \left[ Q_2 (\x, s) Q_2 (\x, s)^\top \right] \right) \a_k &= \a_k^\top (\x \x^\top + D) \a_k^\top \\
    &\leq (\a_k^\top \x)^2 + \| D \|_{\mathrm{op}} \| \a_k \|_2^2 \; ,
\end{align*}
where as before we have that $D$ consists of diagonal elements $\E[Q_2 (\x, s)_i^2]) - (\x)_i^2 = [\mathrm{Var} [Q_2 (\x, s)_i]] \leq M_x^2 / s^2$.
Hence altogether we have
\[
\E [\| \g_k \|_2^2] \leq \| \g'_k \|_2^2 + \frac{A^4 M_x^2}{s^2} \; ,
\]
as claimed.
\end{proof}

\section{Quantizing the Gradients}

Recent work has focused on quantizing the gradients 
with low-precision representation.
We omit the description of this direction
because it is relatively well-studied and is orthogonal
to the contribution of this paper.
From Lemma \ref{lem:quant-facts-2}, we have:

\begin{lemma}
    \label{lem:gradient-quantization}
    Gradient quantization increases the second moment bound of the gradient by a multiplicative $r M^2$ factor. 
\end{lemma}

\section{End-to-end Quantization}

We describe the end-to-end quantization strategy of
quantizing gradients, model, and input samples all 
at the same time. We assume all 3 sources are quantized: Gradient, model and data. However, the update to the model happens in full precision. The gradient becomes:

\begin{eqnarray}
	\g_k := Q_4 \left( Q_1(\a_k, s ) ( Q_2(\a_k, s)^\top Q_3(\x, s) + b_k) , s \right).
\end{eqnarray}
\noindent Here $Q_1, \ldots, Q_4$ are all independent quantizations.  $Q_3$ and  $Q_4$ are normalized with row scaling, and $Q_1, Q_2$ can be normalized arbitrarily.
The iteration then is: 

\begin{eqnarray}
	\x = \x - \gamma \g_k.
\end{eqnarray}

\noindent From combining the previous results, we obtain that, if the samples are normalized, the following holds:

\begin{corollary}
    \label{cor:full-quantization}
    Let $\a_k, \x, b_k$ be so that $\| \a_k \|_2^2 \leq 1, \| \x \|_2^2 \leq R^2$.
    Let $M_a, M_x$ be as above, and let $\g'_k = \a_k (\a_k^\top \x + b_k)$ be the (unquantized) stochastic gradient.
    Then, we have
    \[
    \E [\| \g_k \|_2^2] \leq \rrow \cdot \left( r M_a \left( \| \g'_k \|_2^2 + \frac{R^2}{s^2} \right)  + \frac{r^2 M_a^2 R^2}{s^2} \right) \; .
    \]
\end{corollary}

By a calculation identical to the proof of Cor \ref{cor:var-bound}, we obtain:
\begin{corollary}
    \label{cor:full-quantizationVar}
    Let $\a_k, \x, b_k$ be so that $\| \a_k \|_2^2 \leq 1, \| \x \|_2^2 \leq R^2$.
    Let $M_a, M_x$ be as above, and let $\g'_k = \a_k (\a_k^\top \x + b_k)$ be the (unquantized) stochastic gradient.
    Then, we have
    \[
    \E [\| \g_k - \nabla f(\x_k) \|_2^2] \leq \sigma^2 + \rrow \cdot \left( r M_a \left( \| \g'_k \|_2^2 + \frac{R^2}{s^2} \right)  + \frac{r^2 M_a^2 R^2}{s^2} \right) \; .
    \]
\end{corollary}
Plugging this into Theorem \ref{thm:sgd} gives the bounds for convergence of these end-to-end quantization methods with SGD.

\section{Extension to Classification Models}

\subsection{Least Squares Support Vector Machines}

We first extend our quantization framework to 
least squares support vector machines--a model
popularly used for classification tasks and
often showing comparable accuracy to
SVM~\cite{ye2007svm}.
The  Least Squares SVM optimization problem is formally defined as follows: 

\begin{align*}
\min_{\x}:\quad {1\over 2K}\sum_{k=1}^K (1-b_k\a_k^\top \x)^2 + {c\over 2}\|\x\|^2
\end{align*}

\noindent Without
loss of generality, we assume two-class classification problems, i.e.  $b_k \in \{-1, 1\}$.
We now have:
\begin{align*}
\min_{\x}:\quad {1\over 2K}\sum_{k=1}^K (\a_k^\top \x - b_k)^2 + {c\over 2}\|\x\|^2
\end{align*}
where $c$ is the regularization parameter. The gradient at a randomly selected sample$(\a_k, b_k)$ is: 
\[
\g'_k := \a_k \a_k^\top \x + \a_k b_k + {c\over k} \x.
\]

The gradient is similar to regularized linear regression (Eq.~\ref{eqn:leastsquares-2}). 
In particular, the only difference is the additional $\x$ term.
Since we can quantize this term separately using an additional quantization, and we can quantize first term using the techniques above, we can still use 
the same quantization framework we developed
for linear regression.

\section{Support Vector Machines}

Consider solving the following hinge loss optimization problem for Support Vector Machines(SVM):
\begin{align*}
\min_{\| \x \|_2 \leq R}:\quad \sum_{k=1}^K \max(0, 1 - b_k \a_k^\top \x) \; .
\end{align*}
The (sub-)gradient at a randomly selected sample $(\a_k, b_k)$ is: 

\[
\g'_k := \left\{ \begin{array}{ll}
         -b_k \a_k & \mbox{if $b_k \a_k^\top \x < 1$};\\
         0 & \mbox{otherwise}. \end{array} \right.
\]
Observe that this loss function is not smooth.\footnote{Technically this implies that Theorem \ref{thm:sgd} does not apply in this setting, but other well-known and similar results still do, see \cite{bubeck2015convex}.}
When quantizing samples, the estimator of gradient is biased, as $(1 - b_k \a_k^\top \x)$ and $(1 - b_k Q(\a_k)^\top \x)$ may have different signs, in which case the two procedures will apply different gradients. We say that in this case the gradient is \emph{flipped}. 
We have two approaches to dealing with this: the first provides rigorous guarantees, however, requires some fairly heavy algorithmic machinery (in particular, Johnson-Lindenstrauss matrices with little randomness).
The latter is a simpler heuristic that we find works well in practice.

\subsection{Polynomial approximation and $\ell_2$-refetching via Johnson-Lindenstrauss}
Let $H(x)$ be the Heaviside function, i.e. $H(x) = 1$ if $x \geq 0$ and $H(x) = 0$ if $x < 0$. 
For some fixed parameters $\epsilon, \delta$, we take a degree $d$ polynomial $P$ so that $| P(x) - H(x) | \leq \epsilon$ for all $x \in [-(R^2 + 1), R^2 + 1] \setminus [-\delta, \delta]$, and so that $|P(x)| \leq 1$ for all $x \in  [-(R^2 + 1), R^2 + 1]$.
Since the gradient of the SVM loss may be written as $\g'_k = -H(1 - b_k \a_k^\top \x) b_k \a_k$, we will let $Q$ be a random quantization of $P(1 - b_k \a_k^\top \x)$ (as described in the main paper), and our quantized gradient will be written as $\g_k = -Q(1 - b_k \a_k^\top \x) b_k Q_2 (\a_k)$, where $Q_2$ is an independent quantization of $\a_k$.
We also define
\[
r(s) = \max_{\a_k} \E [\| \g_k \|_2^2] \; 
\]
to be a bound on the second moment of our $\g_k$, for any random choice of $\a_k$.

However, the polynomial approximation offers no guarantees when $1 - b_k \a_k^\top \x \in [-\delta, \delta]$, and thus this provides no black box guarantees on error convergence.
We have two approaches to avoid this problem.
Our first result shows that under reasonable generative conditions, SGD without additional tricks still provides somewhat non-trivial guarantees. However, in general it cannot provide guarantees up to error $\epsilon$, as one would naively hope.
We then describe a technique which always allows us to obtain error $\epsilon$, however, requires refetching.
We show that under the same generative conditions, we do not need to refetch very often.

Throughout this subsection, we will assume that the a spectral norm bound on the second moment of the data points, we should not refetch often.
Such an assumption is quite natural: it should happen for instance if (before rescaling) the data comes from any distribution whose covariance has bounded second moment.
Formally:
\begin{definition}
A set of data points $\a_1, \ldots, \a_m$ is \emph{$C$-isotropic} if $\| \sum_{i = 1}^m \a_i \a_i^\top \| \leq C$, where $\| \cdot \|$ denotes the operator norm of the matrix.
\end{definition}

\subsection{SGD for $C$-isotropic data}
Our first result is the following:
\begin{theorem}
\label{thm:sgd-C}
Suppose the data $\a_i$ is $C$-isotropic, and $\| \a_i \|_2 \leq 1$ for all $i$.
Suppose $\g'_k$ is an unbiased stochastic gradient for $f$ with variance bound $\sigma^2$.
Then $\g_k$ is a $\epsilon + \frac{R}{m C (1 - \delta)^2}$ biased stochastic gradient for $\nabla f (\x)$ with variance bound $\sigma^2 + r(s) + \epsilon^2 +   (r(s) + 4)\frac{R}{m C (1 - \delta)^2} $.
\end{theorem}

In particular, this implies that if $\frac{R}{m C (1 - \delta)^2} = O(\epsilon)$, this bias does not asymptotically change our error, and the variance bound increases by as much as we would expect without the biased-ness of the gradient.
Before we prove Theorem \ref{thm:sgd-C}, we need the following lemma:

\begin{lemma}
\label{lem:C-isotropic}
Suppose $\a_1, \ldots, \a_m$ are $C$-isotropic, and let $\| \x \|_2 \leq R$.
Then, the number of points $L$ satisfying $1 - b_k \a_k \x \in [-\delta, \delta]$ satisfies $L \leq \frac{R}{C(1 - \delta)^2}$.
\end{lemma}
\begin{proof}
Observe that any such point satisfies $(\a_k^\top \x)^2 \geq (1 - \delta)^2$.
Then, by the spectral norm assumption, we have $C \| \x \|_2^2 \geq \sum_{i = 1}^m (\a_i^\top \x)^2 \geq L (1 - \delta)^2$, which implies that $L \leq \frac{R}{C(1 - \delta)^2}$.
\end{proof}

\begin{proof}[Proof of Theorem \ref{thm:sgd-C}]
We first prove the bias bound.
Let $\mathcal{S}$ be the set of points $\a_k$ so that $1 - b_k \a_k \x \in [-\delta, \delta]$.
By the above, we have that $|\mathcal{S}| \leq \frac{R}{C(1 - \delta)^2}$.
Moreover, if $\a_k \not\in \mathcal{S}$, we have by assumption 
\begin{align*}
\| \E_{\g_k} [\g_k | \a_k] - \g'_k \| &= |P(1 - b_k \a_k \x) - H(1 - b_k \a_k \x) | \| \a_k \|_2 \\
 &\leq \epsilon \; .
\end{align*}
Moreover, for any $\a_k$, we always have $\| \E_{\g_k} [\g_k | \a_k] \|_2 \leq  \E_{\g_k} [ \| \g_k \| | \a_k] \leq 1$.
Therefore, we have
\begin{align*}
\left\| \E_{\a_k} \E_{\g_k} [  \g_k ] - \nabla f(\x) \right\| &=  \left\| \E_{\a_k} \E_{\g_k} [  \g_k - \g'_k] \right\|_2 \\
&\leq \frac{1}{m} \left( \sum_{\a_k \not\in \mathcal{S}} \| \E_{\g_k} [  \g_k - \g'_k| \a_k ] \|_2 + \sum_{\a_k \in \mathcal{S}} \| \E_{\g_k} [  \g_k - \g'_k | | \a_k] \|_2  \right) \\
&\leq \frac{1}{m} \left( \epsilon |\mathcal{S}^c|  + |\mathcal{S}| \right) \\
&\leq \epsilon + \frac{R}{m C (1 - \delta)^2} \; .
\end{align*}

We now prove the variance bound.
Observe that if $\a_k \not\in \mathcal{S}$, then 
\begin{align*}
\E [\| \g_k - \g'_k \|_2^2 | \a_k ] &=  \E [\| \g_k - \E[\g_k | \a_k] \|_2^2 | \a_k] + \|  \E[\g_k | \a_k] - \g'_k \|_2^2 \\
&\leq r(s) + \epsilon^2 \; .
\end{align*}
On the other hand, if $\a_k \in \mathcal{S}$, then by the inequality $(a + b)^2 \leq 2a^2 + b^2$ we still have the weaker bound
\begin{align*}
\E [\| \g_k - \g'_k \|_2^2 | \a_k ] &= \E [\| \g_k - \E[\g_k | \a_k] \|_2^2 | \a_k] + \|  \E[\g_k | \a_k] - \g'_k \|_2^2 \\
&\leq r(s) + 2 \E [\| \g_k \|_2^2 | \a_k] + 2 \| \g'_k \|_2^2 \\
&\leq r(s) + 4 \; ,
\end{align*}
since $\| \g_k \|_2^2 \leq \| \a_k \|_2^2 \leq 1$ and similarly for $\g'_k$.
Thus, we have 
\begin{align*}
\E [\| \g_k - \nabla f(\x) \|_2^2] &= \sigma^2 + \E [\| \g_k - \g'_k \|_2^2 ] \\
&= \sigma^2 + \frac{1}{m} \left( \sum_{\a_k \not\in \mathcal{S}} \| \E_{\g_k} [ \| \g_k - \g'_k \|_2^2 | \a_k] \|_2 + \sum_{\a_k \in \mathcal{S}} \| \E_{\g_k} [  \| \g_k - \g'_k \|_2^2 | | \a_k \in \mathcal{S}] \|_2  \right) \\
&\leq \sigma^2 + \frac{1}{m} \left( (r(s) + \epsilon^2) \cdot |\mathcal{S}^c| + (r(s) + 4) \cdot | \mathcal{S} | \right) \\
&\leq \sigma^2 + r(s) + \epsilon^2 +   (r(s) + 4)\frac{R}{m C (1 - \delta)^2} \; ,
\end{align*}
as claimed.
\end{proof}

\subsection{$\ell_2$-refetching}

One drawback of the approach outlined above is that in general, if $\frac{R}{m C (1 - \delta)^2}$ is large, then this method does not provide any provable guarantees.
In this section, we show that it is still possible, with some additional preprocessing, to provide non-trivial guarantees in this setting, without increasing the communication that much.

Our approach will be to estimate this quantity using little communication per round, and then refetch the data points if $1 - b_k \a_k^\top \x \in [-\delta, \delta]$.
We show that under reasonable generative assumptions on the $\a_k$, we will not have to refetch very often.

\subsubsection{$\ell_2$-refetching using Johnson-Lindenstrauss}
For scalars $a, b$ and $\gamma \in [0, 1)$, we will use $a \leq_\gamma b$ to mean $a \leq e^\gamma b$, and $a \approx_\gamma b$ to denote that $e^{-\gamma} a \leq b \leq e^\gamma a$.

We require the following theorem:
\begin{theorem}
\label{thm:JL}
Fix $\gamma, \tau > 0$, $n$.
Then, there is a distribution $\mathcal{D}$ over $n \times r$ matrices which take values in $\pm 1$ so that if $M$  is drawn from $\mathcal{D}$, then for any $x \in \R^n$, we have $\| x \|_2 \approx_\gamma \| M x \|_2$ with probability $1 - \tau$.
If the processors have shared randomness, the distribution can be sampled in time $O(n r)$.

Otherwise, the distribution can be sampled from in time $O(n \log n + \poly (r))$, and using only 
\[
\alpha (n, \gamma, \tau) := O \left( \log n + \log (1 / \tau) \cdot \log \left( \frac{\log 1 / \tau}{\gamma} \right) \right)
\]
bits of randomness.
\end{theorem}
\noindent
If $M \sim \mathcal{D}$, we will call $M$ a \emph{JL matrix}.
In the remainder, we will assume that the processors have shared randomness, for instance by using a pseudo-random number generator initialized with the same seed.
We will use this shared randomness solely to sample the same $M$ between the two processors.
Otherwise, one processor may simply send $\alpha (n, \gamma, \tau)$ random bits to the other, and it is easily verified this does not change the asymptotic amount of communication required.

As a corollary of this, we have:
\begin{corollary}
\label{cor:l2-refetch}
Fix $\delta, \tau$. Suppose one processor has $\a_k$ and another has $\x$.
Suppose furthermore that $\| \a_k \|_2^2 \leq 1$, and $\| \x \|_2^2 \leq R^2$, where $R \geq 1$.
 There is a protocol which with probability $1 - \tau$ outputs a $c$ so that $|c - (1 - b_k \a_k^\top \x_k)| \leq \delta$ that requires each processor to send 
 \[
 O\left( R^2 \frac{\log (1 / \tau) \log (n / \delta)}{\gamma^2} \right) \text{  bits.}
 \]
\end{corollary}
\begin{proof}
The protocol is as follows.
Let $\gamma' = O(\gamma / R)$, and let $r$ be as above, except with $\gamma'$.
Using these shared random bits, have both processors sample the same $M \sim \mathcal{D}$.
Then, have the processors send $M \a_k$ and $M \x$ up to $O(\log n / \delta)$ bits of precision per coordinate.
Using these vectors, compute the quantities $\| M (\a_k - \x) \|_2^2, \| M \a_k \|_2^2, \| M \x \|_2^2$ up to additive error $O(\delta)$.
Then, output 
$c = 1 - b_k (\| M \a_k - M \x \|_2^2 - ( \| M \a_k \|_2^2 + \| M \x \|_2^2 ))$.

That this protocol sends the correct number of bits follows from the description of the algorithm.
Thus it suffices to prove correctness.
By Theorem \ref{thm:JL} and a union bound, we have that $\| M (\a_k - \x) \|_2^2 \approx_{2 \gamma} \| \a_k - \x \|_2^2$, $ \| M \a_k \|_2^2 \approx_{2 \gamma} \| \a_k \|_2^2, \| M \x \|_2^2 \approx_{2 \gamma} \| \x \|_2^2$ with probability $1-\tau$.
Let us condition on this event for the rest of the proof.
We have $\| x \|_2^2 \leq R^2$ and so by a triangle inequality, $\| \a_k - x \|_2^2 \leq (\sqrt{R} + 1)^2$.
Thus, by our choice of $\gamma$, we have that $|\| M v \|_2^2 - \| v \|_2^2 | \leq O(\delta)$, for all $v \in \{\a_k - x, \a_k, \x \}$.
Thus, since $2 \a_k^\top \x = \| \a_k - \x \|_2^2 - ( \| \a_k \|_2^2 + \| \x \|_2^2 )$, this implies that by an appropriate setting of parameters, we have that $|c - (1 - b_k \a_k^\top \x_k)| \leq \delta$, as claimed.
\end{proof}

Thus, our algorithm for computing the gradient for SVM is as follows.
\begin{itemize}
\item
Use the protocol from Corollary \ref{cor:l2-refetch} with $\tau = O(\delta)$ to compute a $c$  so that with probability $1 - \delta$ we have $|c - (1 - b_k \a_k^\top \x) | \leq \delta$.
\item
If $|c| \leq 2 \delta$, we refetch and compute the full (unquantized) gradient $\vec{g}_k'$.
\item
Otherwise, we output the polynomial approximation $\g_k$.
\end{itemize}
Then, our result is the following:
\begin{theorem}
Let $\hat{\g}_k$ be the estimator produced by the above procedure.
Assume that $\| \a_k \|_2 \leq 1$ for all $k$.
Then, $\hat{\g}_k$ is a stochastic gradient for $\nabla f(\x)$ with bias $\epsilon$, and with variance $\sigma^2 + \delta + r(s)$.
\end{theorem}
\begin{proof}
We first prove the bias bound.
Fix any choice of $k$.
Let us say the above procedure \emph{succeeds} if the estimator $c$ it produces satisfies $|c - (1 - b_k \a_k^\top \x) | \leq \delta$, and say it fails otherwise.

\noindent There are two cases, depending on $c$.
Suppose we have $c \in [-2 \delta, 2 \delta]$.
Then, we have
\begin{align*}
\E [\hat{\g}_k \mid c \in [-2 \delta, 2 \delta]] &= \E [{\g}_k] = \g'_k \; ,
\end{align*}
so obviously in this case we are unbiased.

\noindent On the other hand, if $c \not\in [-2 \delta, 2 \delta]$, then we have
\begin{equation}
\label{eq:l2-refetch}
\E [\hat{\g}_k \mid c \not\in [-2 \delta, 2 \delta]] = (\g'_k + \w_k) \Pr [c \not\in [-2 \delta, 2 \delta], \mathrm{success}] + \g_k  \Pr [c \not\in [-2 \delta, 2 \delta], \mathrm{failure}] \; ,
\end{equation}
where $\| \w_k \|_2 \leq O(\delta)$, since if $c \not\in [-2 \delta, 2 \delta]$ and we succeed, this implies that $1 - b_k \a_k^\top \x_k \not\in [\delta, \delta]$, and thus in this case 
\begin{align*}
\left\| \E [\hat{\g}_k \mid c \not\in [-2 \delta, 2 \delta], \mathrm{success}] - \g'_k \right\| &= \left\| \left( \E [Q(1 - b_k \a_k^\top \x)] - H(1 - b_k \a_k^\top \x) \right)  (- b_k \a_k) \right\| \\
&\leq \left| \E [Q(1 - b_k \a_k^\top \x)] - H(1 - b_k \a_k^\top \x)) \right| \\
&= \left| P(1 - b_k \a_k^\top \x) - H(1 - b_k \a_k^\top \x ) \right| \\
&\leq O(\delta) \; ,
\end{align*}
by assumption.

Finally, since $\Pr [c \not\in [-2 \delta, 2 \delta], \mathrm{failure}] \leq O(\delta)$, and $\| \vec{g}_k \|_2 = \| \a_k \|_2 \leq 1$ by assumption, (\ref{eq:l2-refetch}) implies that $\| \E [\hat{\g}_k] - \g'_k \|_2 \leq O(\delta)$.
By an appropriate choice of constants, this implies the desired bias bound for any fixed $\a_k$.
Taking an expectation over all $\a_k$ yields the desired result.

We now turn our attention to the variance bound.
We again split into two cases, depending on $c$.
Clearly, we have
\[
\E [\| \hat{\g}_k - \nabla f(\x) \|_2^2 \mid c \in [-2 \delta, 2 \delta]] = \sigma^2 \; .
\]
The interesting case is when $c \not\in [-2 \delta, 2 \delta]$.
In this case, we have
\begin{align*}
\E [\| \hat{\g}_k - \nabla f(\x) \|_2^2 \mid c \not\in [-2 \delta, 2 \delta]] &= \E [\| \hat{\g}_k - \nabla f(\x) \|_2^2 \mid c \not\in [-2 \delta, 2 \delta], \mathrm{success}] \Pr [\mathrm{success}]  \\
 &~~~~ +\E [\| \hat{\g}_k - \nabla f(\x) \|_2^2 \mid c \not\in [-2 \delta, 2 \delta], \mathrm{failure}] \Pr [\mathrm{failure}] \; ,
\end{align*}
as before.
We analyze each term separately.
As before, observe that if $c \not\in [-2 \delta, 2 \delta]$ and we succeed, then $1 - b_k \a_k^\top \x \not\in [- \delta, \delta]$.
Hence, we have
\begin{align*}
\E [\| \hat{\g}_k - \nabla f(\x) \|_2^2 \mid c \in [-2 \delta, 2 \delta], \mathrm{success}] &= \E_{\g_k} \E_{\a_k} [\| \g_k - \nabla f(\x) \|_2^2 \mid c \in [-2 \delta, 2 \delta], \mathrm{success}] \\
&= \sigma^2 + \E_{\a_k} \left[ \E_{\g_k} [\| \g_k - \g'_k \|_2^2 \mid c \in [-2 \delta, 2 \delta], \mathrm{success}, \a_k] \right] \\
&\sigma^2 + \E_{\a_k} \left[ \| \w_k \|_2^2 + \E_{\g_k} [\| \g_k \|_2^2 \mid c \in [-2 \delta, 2 \delta], \mathrm{success}, \a_k] \right] \\
&\leq \sigma^2 + \delta^2 + r(s) \; .
\end{align*}
Moreover, since $(a + b)^2 \leq 2 a^2 + 2 b^2$, we have
\begin{align*}
\E [\| \hat{\g}_k - \nabla f(\x) \|_2^2 \mid c \in [-2 \delta, 2 \delta], \mathrm{failure}] &\leq \E [\| \hat{\g}_k \|_2^2 \mid c \in [-2 \delta, 2 \delta], \mathrm{failure}] + 2 \\
&\leq r(s) + 2 \; .
\end{align*}
Hence the variance if $c \not\in [-2 \delta, 2 \delta]$ is upper bounded by 
\[
\E [\| \hat{\g}_k - \nabla f(\x) \|_2^2 \mid c \not\in [-2 \delta, 2 \delta]] \leq \sigma^2 + \delta^2 + r(s) + \delta (1 + r(s)) \; ,
\]
which simplifies to the claimed bound.
%
\end{proof}

\subsubsection{A bound on the refetching probability}
We now show that under a reasonable generative model, we will not have to refetch very often.
Under this assumption, we show:
\begin{theorem}
Suppose $\a_k$ are $C$-isotropic.
Then, the probability we refetch at any iteration under the $\ell_2$-refetching scheme is at most $\frac{R}{n C(1 - \delta)^2} + O(\delta)$.
\end{theorem}
\begin{proof}
Fix any $\x$ with $\| \x \|_2 \leq R$.
By Lemma \ref{lem:C-isotropic}, the number of points with $1 - b_k \a_k^\top \x \in [-3 \delta, 3\delta]$ is at most $\frac{R}{C(1 - \delta)^2} $.
If we succeed, and $1 - b_k \a_k^\top \x \not\in [-3 \delta, 3\delta]$, then by definition we do not refetch, the probability we refetch is bounded by the sum of the probability we choose a point with $1 - b_k \a_k^\top \x \in [-3 \delta, \delta]$ and the probability of failure.
By the above, this is bounded by $\frac{R}{n C(1 - \delta)^2} + O(\delta)$, as claimed.
\end{proof}

\subsection{$\ell_1$-refetching}
A simpler algorithmic to ensure we do not have any gradient flips is as follows.
After getting the quantized sample $Q(\a_k)$, we can compute  upper and lower bounds on $1 - b_k \a_k^\top \x$. 
The upper bound is given by:
\[1 - b_k Q(\a_k)^\top \x + \frac{\| \x \|_1}{s} \; ,
\]
and the lower bound is given by:
\[1 - b_k Q(\a_k)^\top \x - \frac{\| \x \|_1}{s} \; ,
\]
where $1/s$ is ``resolution'' of the quantization.
If the upper and lower bounds of a quantized sample have the same sign, then we can be certain that  no ``flipping'' will occur, and we can use the quantized sample. otherwise we send a request to fetch the original data and use it to  compute the gradient.
This seems to work well in practice, however, we could not prove any guarantees about how often we refetch with this guarantee, under any reasonable generative models.


\section{Optimal Quantization Strategy}

We prove Lemma~\ref{lem:discrete} and Theorem~\ref{thm:optQ} in the main paper here.

\paragraph*{Problem Setting.}
Assume a set of real numbers $\Omega = \{x_1, \ldots, x_N\}$ with cardinality $N$. WLOG, assume that all numbers are in $[0, 1]$ and sorted are sorted such that $x_1 \leq \ldots \leq x_N$. 

The goal is to find an partition $\setI = \{I_j\}_{j = 1}^s$ of $[0, 1]$ into $s$ disjoint intervals, so that if we randomly quantize every $x \in I_j$ to an endpoint of $I_j$, the variance is minimal over all possible partitions of $[0, 1]$ into $k$ intervals.
Formally:
\begin{align}
\nonumber \min_{\setI: |\setI| = s} \quad & \mathcal{MV}(\setI) := {1\over N}\sum_{j = 1}^k \sum_{x_i \in I_j} \err(x_i, I_j)\\
\text{s.t.}\quad & \bigcup_{j = 1}^s I_j = [0, 1],\quad I_j\cap l_k = \emptyset~\text{for $k\neq j$}
\label{eq:opt_Q-2}
\end{align}
where $\err (x, I) = (b - x) (x - a)$ is the variance for point $x \in I$ if we quantize $x$ to an endpoint of $I = [a, b]$.\label{eq:var}
That is, $\err (x, I)$ is the variance of the (unique) distribution $D$ supported on ${a, b}$  so that $\E_{X \sim D} [X] = x$.


Given an interval $I \subseteq [0, 1]$, we let $\setX_I$ be the set of $x_j \in \setX$ contained in $I$.
We also define $\err (\setX, I) = \sum_{x_j \in I} \err (x_j, I)$.
Given a partition $\setI$ of $[0, 1]$, we let $\err (\setX, \setI) = \sum_{I \in \setI} \err (\setX, I)$.
We let the optimum solution be $\setI^* = \argmin_{|\setI| = k} \err (\setX, \setI)$, breaking ties randomly. 

\begin{lemma*}
\label{lem:discrete-2}
There is an $\setI^*$ so that all endpoints of any $I \in \setI^*$ are in $\Omega \cup \{0, 1\}$.
\end{lemma*}

\begin{proof}
Fix any endpoint $b$ of intervals in $\setI^*$. WLOG assume that $b \neq 0, 1$. Then we must have $I = [a, b]$ and $I' = [b, c]$ for some $I, I' \in \setI^*$. Observe that the choice of $b$ only affects the error for points in $I \cup I'$. We have that $\err (\Omega, I) + \err (\Omega, I') $ is given by 
\begin{align*}
& \sum_{x \in I} (b - x) (x - a) + \sum_{x \in I'} (c -x)(x - b) \\
&= A b + C \; ,
\end{align*}
where $A, C$ are constants which do not depend on $b$. Hence, this is a linear objective in $b$. Since $b$ can freely range between the rightmost point in $I$ and the leftmost point in $I'$, there is an optimizer for this solution at one of those two points. Hence we may choose $b \in \Omega$.
\end{proof}

Therefore, to solve the problem in an exact way, we just need to select a subset of data points in $\Omega$ as quantization points. Define $T(k, m)$ be the optimal total variance for points in $[0, d_m]$ with $k$ quantization levels choosing $d_m=x_m$ for all $m=1,2,\cdots, N$. Our goal is to calculate $T(s, N)$. This problem can be solved by dynamic programing using the following recursion
\[
T(k, m) = \min_{j\in \{k-1, k, \cdots, m-1\}} T(k-1,j) + V(j,m),
\]
where $V(j,m)$ denotes the total variance of points falling into $[d_j, d_m]$. The complexity of calculating the matrix $V(\cdot, \cdot)$ is $O(N^2 + N)$ and the complexity of calculating matrix $T(\cdot, \cdot)$ is $O(kN^2)$. The  memory cost is $O(kN + N^2)$. 

\subsection{Heuristics}

The exact algorithm has a complexity that is quadratic to the number of data points. To make our algorithm practical,
we develop an approximation algorithm that only needs to scan all data points once and has linear complexity to $N$.

\paragraph*{Discretization.}

We can discretize the range $[0,1]$ into $M$ intervals, i.e., $[0,d_1), [d_1, d_2), \cdots, [d_{M-1}, 1]$ with $0< d_1<d_2<\cdots < d_{M-1}<1$. We then restrict our algorithms to only choose $k$ quantization points within these $M$ points, instead of all $N$ points in the exact algorithm. The following result bounded the quality of this approximation.

\begin{theorem*} \label{thm:optQ-2}
Let the maximal number of data points in each ``small interval'' (defined by $\{d_m\}_{m=1}^{M-1}$) and the maximal length of small intervals be bounded by $bN/M$ and $a/M$, respectively. Let ${\mathcal{I}^*} := \{l^*_j\}_{k=1}^{k-1}$ and $\hat{\mathcal{I}}^* :=\{\hat{l}^*_k\}_{k=1}^{k-1}$ be the optimal quantization to \eqref{eq:opt_Q-2} and the solution with discretization. Let $cM/k$ be the upper bound of the number of small intervals crossed by any ``large interval'' (defined by ${\mathcal{I}}^*$). Then we have the discretization error bounded by
\[
 \mathcal{MV}(\hat{\mathcal{I}}^*) -  \mathcal{MV}({\mathcal{I}}^*) \leq {a^2b k \over 4 M^3} + {a^2bc^2 \over Mk}.
\]
\end{theorem*}

\begin{proof}
Let $p^*_0$ be $0$ and $p^*_K=1$.We quantitize $\{p^*_k\}_{k=1}^{K-1}$ one element by one element, while monitor the changing of the total variance $N \times \mathcal{MV(\cdot)}$. We first quantize $p^*_1$ to the closest value (denoted it by $Q(p^*_1)$) in $\{d_m\}_{m=1}^{M-1} \cup \{p^*_0,p^*_K\}$ under the monotonicity constraint, that is, $p^*_0\leq Q(p^*_1) \leq p^*_2$. Here, one important observation is $|p^*_{1} - Q(p^*_1)|\leq a/M$. Consider the total variance of this new solution $Q(p^*_1), p^*_2, \cdots, p^*_{K-1}$. The variance of points falling into the range $[p^*_2, 1]$ does not change at all. Without the loss of generality, assume $p^*_{1} \geq Q(p^*_1)$.

Next we consider points falling into the following three sets $C_1 = [p^*_0, Q(p^*_1)]$, $C_2 = [Q(p^*_1), p^*_1]$, and $C_3 = [p^*_1, p^*_2]$. The variance of points of falling into $C_1$ gets reduced from the form of variance in \eqref{eq:var}. Next we only need to check the variance change for points in $C_2$ and $C_3$. Consider $C_2$ first. The variance for point $x$ in $C_2$ is
\[
(x-Q(p^*_1))(p^*_1 - x) \leq {a^2 \over 4 M^2}. 
\]  
Thus, the change of variance for points in $C_2$ would be bounded by ${a^2 \over 4 M^2}$. Then consider $C_3$. The change of variance for point $x$ in $C_3$ is
\begin{align*}
& (x-Q(p^*_1))(p^*_2 - x) - (x-p^*_1) (p^*_2 - x) 
\\
= & (p^*_1 - Q(p^*_1)) (p^*_2 - x)
\\
\leq & {a\over M}(p^*_2 - x)
\end{align*}
Therefore, the change of total variance from $\{p^*_1, p^*_2, \cdots, p^*_{K-1}\}$ to $\{Q(p^*_1), p^*_2, \cdots, p^*_{K-1}\}$ is bounded by 
\begin{align}
\nonumber
& \sum_{x\in C_2} {a^2 \over 4M^2} + \sum_{x \in C_3} {a\over M}(p^*_2 - x)
\\
\nonumber
\leq & {Nb \over M} {a^2 \over 4M^2} + {a\over M}{Nb \over M} \sum_{t=1}^{cM/K}t{a\over M}
\\
\leq &
{a^2b N \over 4 M^3} + {a^2bc^2 N \over MK^2}
\label{eq:bound_1step}
\end{align}
Similarly, we quantitize $p^2_2$ in $\{Q(p^*_1), p^*_2, \cdots, p^*_{K-1}\}$ to get a new solution $\{Q(p^*_1), Q(p^*_2), \cdots, p^*_{K-1}\}$ while maintain the monotonicity. We can establish the same upper bound to \eqref{eq:bound_1step}. Following this idea, we can obtain a quantization solution $\{Q(p^*_1), Q(p^*_2), \cdots, Q(p^*_{K-1})\}$. Therefore, we obtain that 
\begin{align*}
&  \mathcal{MV}(Q(p^*_1), Q(p^*_2), \cdots, Q(p^*_{K-1})) -  \mathcal{MV}({p}^*_1,\cdots, {p}^*_{K-1}) 
 \\
 & \quad \leq {a^2b K \over 4 M^3} + {a^2bc^2 \over MK}.
\end{align*}
Using the fact that $ \mathcal{MV}(p^*_1,\cdots, p^*_{K-1})$ is smaller than $\mathcal{MV}(Q(p^*_1), Q(p^*_2), \cdots, Q(p^*_{K-1}))$ proves the claim.
\end{proof}
Theorem~\ref{thm:optQ} suggests that the mean variance using the discrete variance-optimal quantization will converge to the optimal with the rate $O(1/Mk)$.

\section{A Greedy Algorithm for Finding Quantization Points}
\subsection{Setup}
We have $n$ points $\setX = x_1, \ldots, x_n \in [0, 1]$.
Our goal is to partition $[0, 1]$ into $k$ intervals $I_1, \ldots, I_k$, so that if we quantize all $x_i$ in $I_j$ to an endpoint of $I_j$, we minimize the variance.
If $I = [a, b]$, and $x \in I$, it is not hard to show that the variance of the quantization is given by $\err (x, I) = (b - x) (x - a)$.

\paragraph{Notation.} Given an interval $I \subseteq [0, 1]$, we let $\setX_I$ be the set of $x_j \in \setX$ contained in $I$.
We also define $\err (\setX, I) = \sum_{x_j \in I} \err (x_j, I)$.
Given a partition $\setI$ of $[0, 1]$, we let $\err (\setX, \setI) = \sum_{I \in \setI} \err (\setX, I)$.
We also let $\setI^* = \argmin_{|\setI| = k} \err (\setX, \setI)$ (if there are multiple, then choose one arbitrarily), and we let $\OPT_k = \err(\setX, \setI^*)$.

We require the following lemma, whose proof is trivial and omitted.
\begin{lemma}
\label{lem:subset}
If $I \subseteq I'$, then $\err(\setX, I) = \err (\setX_I, I) \leq \err (\setX_I, I')$.
\end{lemma}

\subsection{A nearly linear time algorithm for nearly minimizing the error}
First, we observe that it is trivial that the optimal partition must have endpoints solely at points in $\setX$.
Thus we may restrict our attention to such partitions.
The algorithm is given in Algorithm \ref{alg:adaquant}.
The algorithm is inspired by  greedy merging algorithms for histogram recovery.

\begin{algorithm}[htb]
\begin{algorithmic}[1]
\Function{AdaQuant}{$\setX, k, \gamma, \delta$}
\State Let $\setI = [0, 1]$ be a partition of $[n]$, initially with one breakpoint at each point in $\setX \cup \{0, 1\}$.
\While{$|\setI| > 2(1 + \gamma) k + \delta$}
    \State Pair up consecutive intervals in $\setI$ to form $\setJ$
    \For{each $I \in \setJ$}
        \State Let $e_I = \err (\setX, I)$
    \EndFor
    \State Let $\setJ_1$ be the set of $(1 + \gamma) k$ intervals $I \in \setI$ with largest $e_I$.
    \For{each $I \in \setJ_1$}
        \State Let $I = I_1 \cup I_2$, where $I_1, I_2 \in \setI$
        \State Remove $I$ from $\setJ$
        \State Insert $I_1, I_2$ into $\setJ$
    \EndFor
    \State Let $\setI \gets \setJ$
\EndWhile
\State \textbf{return} the partition with endpoints at $\setI$.
\EndFunction
\end{algorithmic}
\caption{Nearly-linear time algorithm for finding quantization points}
\label{alg:adaquant}
\end{algorithm}

We first show this algorithm runs in nearly linear time:
\begin{theorem}
Given any $\setX, k, \gamma, \delta$, we have that $\textsc{AdaQuant} (\setX, k, \gamma, \delta)$ runs in time $O(n (\log (n / \gamma) + \log \log 1 / \delta))$ 
\end{theorem}

Our main contribution is to show that the algorithm does indeed achieve good error:
\begin{theorem}
Given any $\setX, k, \gamma, \delta$, let $\setI$ be the output of $\textsc{AdaQuant} (\setX, k, \gamma, \delta)$.
Then we have that $\err (\setX, \setI) \leq \left( 1 + \frac{1}{\gamma} \right) \OPT_k$.
\end{theorem}
\begin{proof}
Partition $\setI = \setF \cup \setJ$, where $\setF$ is the set of intervals $I \in \setI$ so that $I \subseteq I'$ for some $I' \in \setI^*$, and let $\setJ$ be the remaining intervals.
Observe that by a simple counting argument, we have that $|\setJ| \leq k$.
By Lemma \ref{lem:subset}, we have that 
\[
\sum_{I \in \setF} \err(\setX, I) \leq \OPT_k \; .
\]
We now seek to bound the error along intervals in $\setJ$.
Let $I \in \setJ$.
It must have been merged in some iteration of \textsc{AdaQuant}.
Therefore, in that iteration, there must have been $(1 + \gamma)k$ merged intervals $J_1, \ldots, J_{(1 + \gamma) k}$ so that $\err (\setX, I) \leq \err (\setX, J_\ell)$, for all $\ell = 1, \ldots, (1 + \gamma) k$.
By a simple counting argument, at most $k$ of the $J_\ell$ are not contained in some interval in $\setI^*$.
WLOG, assume that $J_1, \ldots, J_{\gamma \ell}$ are all contained in some interval of $\setI^*$.
By Lemma \ref{lem:subset}, we have that $\sum_{j = 1}^{\gamma \ell} \err (\setX, J_j) \leq \OPT$.
In particular, this implies that $\err (\setX, I) \leq \OPT / (\gamma k)$.
Since this holds for all $I \in \setJ$, and $|\setJ| \leq k$, this implies that 
\[
\sum_{I \in \setJ} \err (\setX, I) \leq k \frac{\OPT_k}{\gamma k} \leq \frac{1}{\gamma} \OPT_k \; .
\]
Combining the above two expressions yields that $\err (\setX, \setI) \leq  \left( 1 + \frac{1}{\gamma} \right) \OPT_k$, as claimed.
\end{proof}

\section{Extended Experiments}

This section provides more experiment results.
All experiments settings are the same with
the previous section.

\subsection{Linear Models}

For linear models, we validate that 
with double sampling, SGD with low
precision converges---in
comparable empirical 
convergence rates---to the same solution
as SGD with full precision and we want to
understand how many bits do we need to
achieve it empirically and how it is related to 
the size of dataset.

\begin{figure}[t]
\centering
    \includegraphics[width=1\columnwidth]{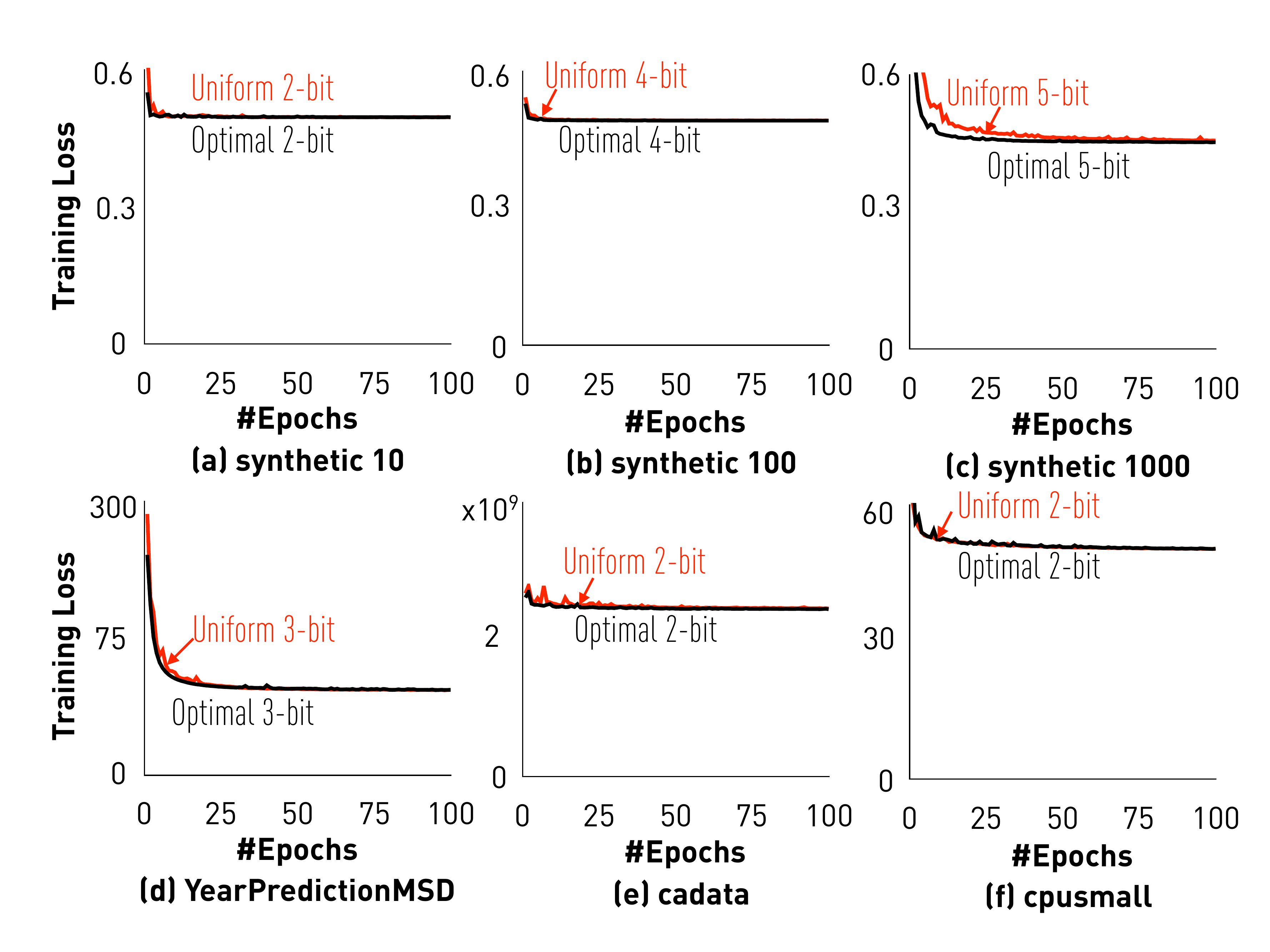}
    
\caption{Linear regression with \emph{end-to-end quantization} on multiple datasets}
\label{fig:lr-endtoend}
\end{figure}
\begin{figure}[t]
\centering
    \includegraphics[width=0.6\columnwidth]{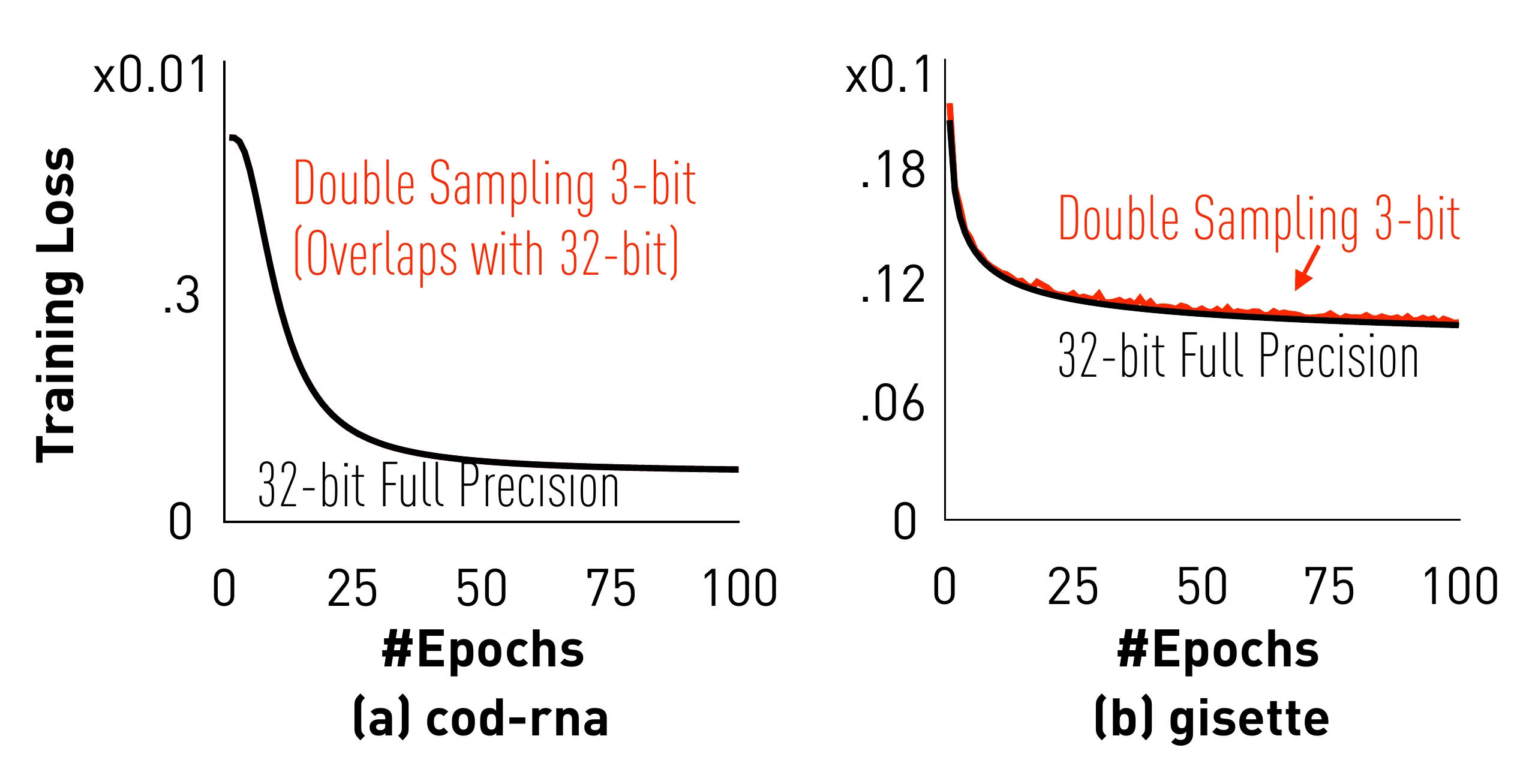}
    
\caption{Least squares SVM with \emph{end-to-end quantization} on multiple datasets}
\label{fig:lssvm-endtoend}
\end{figure}

Figure~\ref{fig:lr-endtoend} and~\ref{fig:lssvm-endtoend} illustrates
the result of training linear models: 
linear regression and least squares SVMs, respectively,
with end-to-end low-precision and 
full precision. For
low precision, we pick the 
smallest number of bits that
results in a smooth convergence
curve. We compare the final 
training loss in both settings 
and the convergence rate.

We see that, for both linear regression 
and least squares SVM,
on all our datasets,
using 5- or 6-bit is always enough
to converge to the same solution
with comparable convergence rate. 
This validates our prediction that
double-sampling provides an
unbiased estimator of the gradient.
Considering the size of input
samples that we need to read, we
could potentially save 6--8$\times$ 
memory bandwidth compared to using 
32-bit. 

We also see from Figure~\ref{fig:lr-endtoend} (a) to (c) that if the dataset has more
features, usually we need more bits for quantization,
because the variance induced by quantization
is higher when the dimensionality is higher.

\subsection{Non-Linear Models}

Figure~\ref{fig:refetch} illustrates
the result of training SVM
with refetching heuristic. 
We see that, 
with our refetching heuristic,
SGD converges to similar training loss 
with a comparable empirical convergence 
rate for SVM. If we increase the number of
bits we use, we need to refetch less data and
if we use 8-bit quantization, we only need to fetch
about $6\%$ of data.
We also experience no loss in test accuracy.

\begin{figure}[t]
\centering
    \includegraphics[width=0.6\columnwidth]{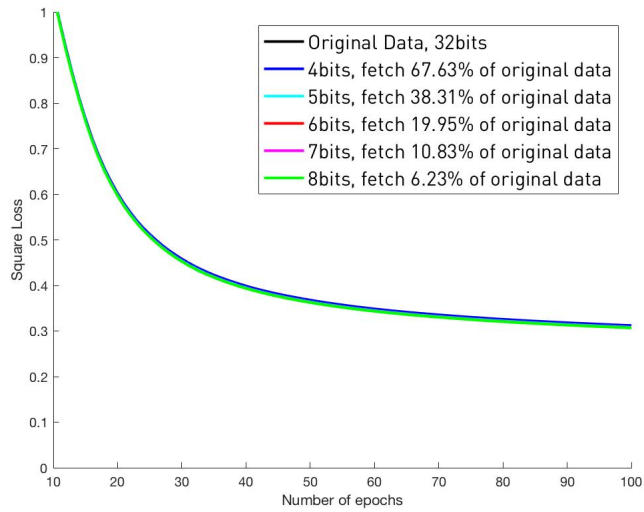}
    
\caption{SVM with low precision data and refetching on cod-rna dataset}
\label{fig:refetch}
\end{figure}

\section{Implementation on FPGA}
\label{section:implementation}
We show the computation pipeline for full-precision SGD and quantized SGD implemented on FPGA in Figure~\ref{fig:floatFPGASGD} and~\ref{fig:qallFPGASGD}. For detailed FPGA implementation, please refer to~\cite{kara2017fpga}.

\begin{figure}[t]
\centering
\includegraphics[width=.8\columnwidth]{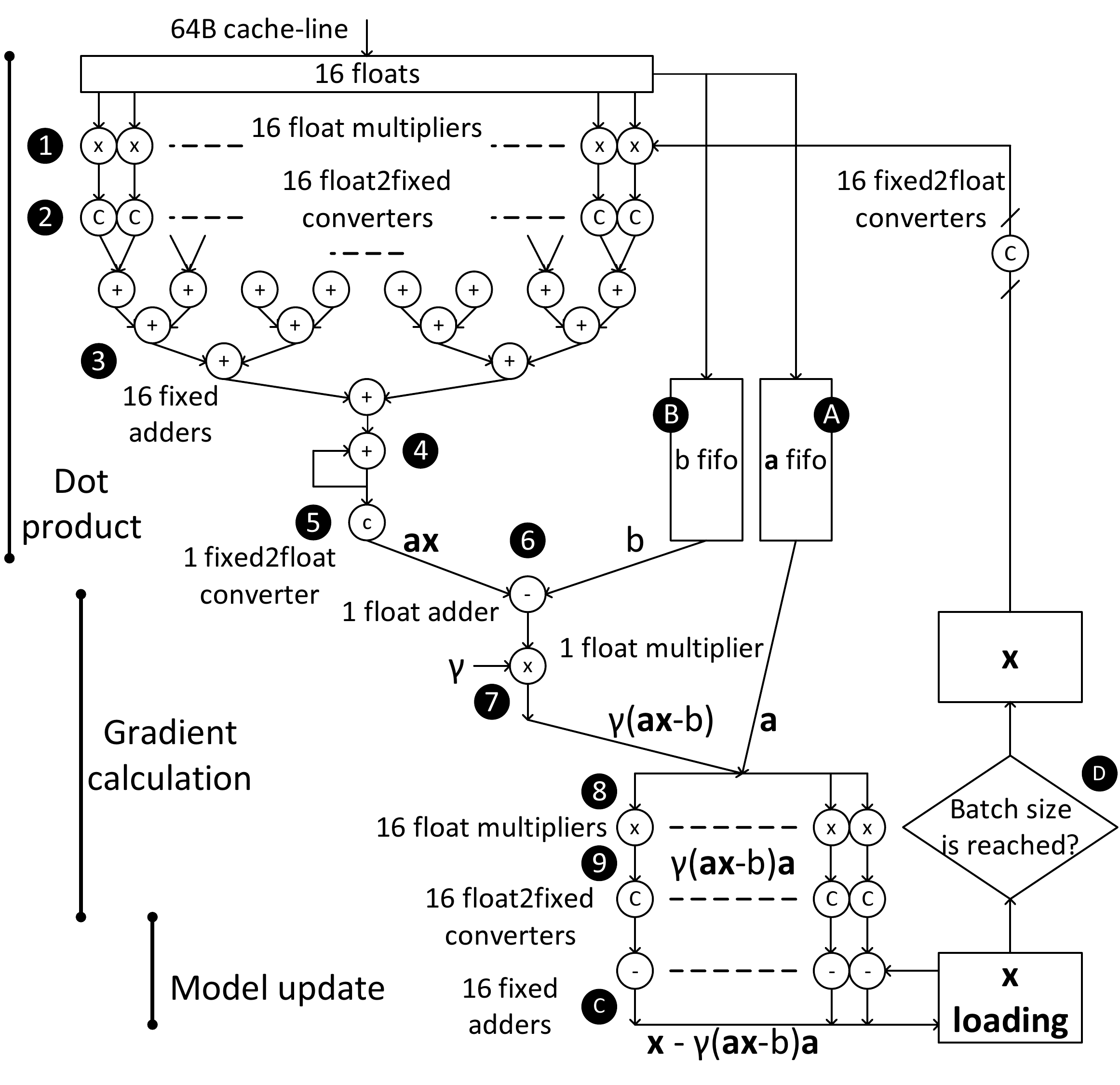}
\caption{Computation pipeline for \textit{float} FPGA-SGD, with a latency of 36 cycles, a data width of 64B and a processing rate of 64B/cycle.}
\label{fig:floatFPGASGD}
\end{figure}

\begin{figure*}[t]
\centering
\begin{subfigure}[t]{.6\columnwidth}
\centering
\includegraphics[width=\columnwidth]{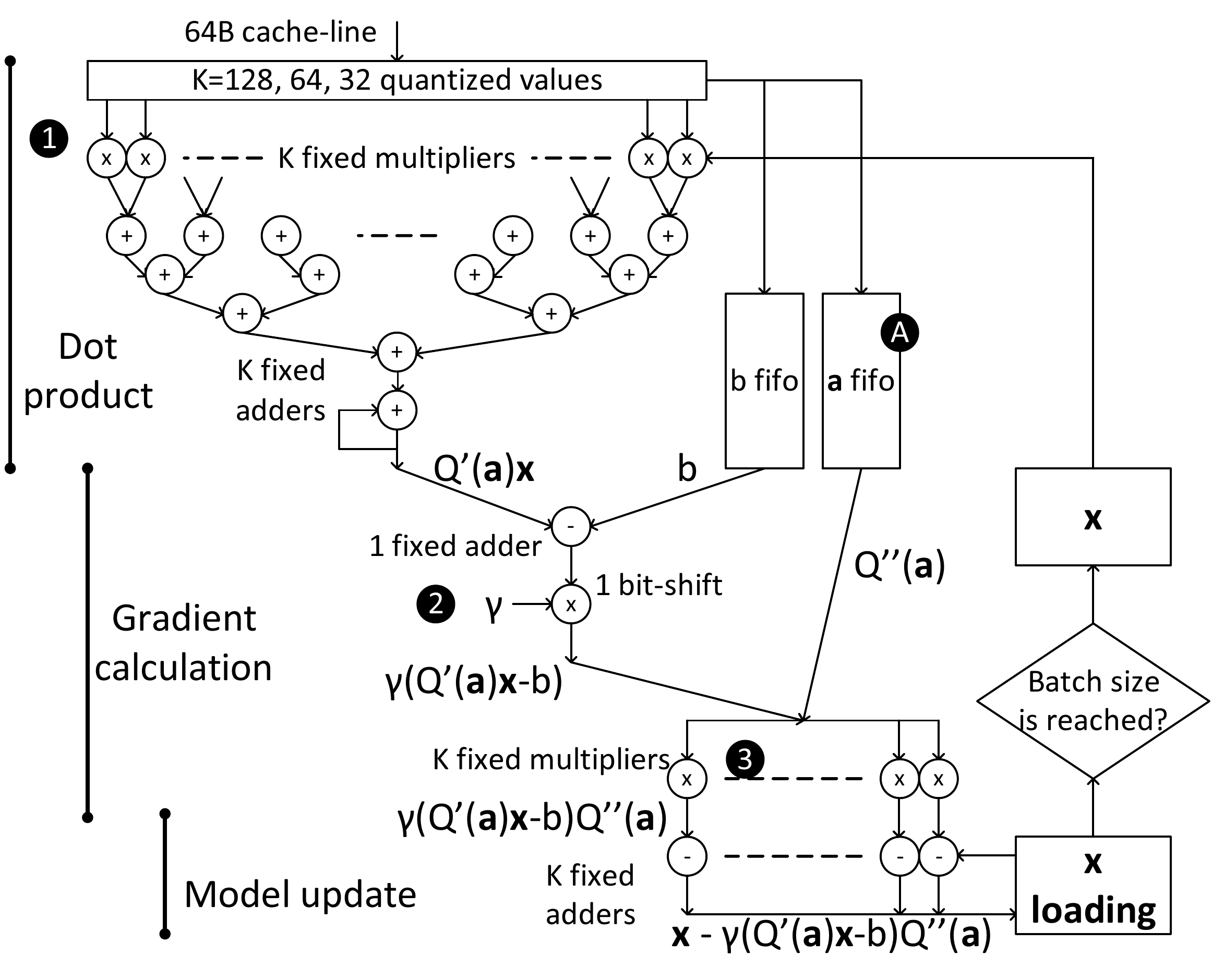}
\caption{\textit{Q2}, \textit{Q4} and \textit{Q8} FPGA-SGD, with a latency of $log(K)$+5 cycles, a data width of 64B and a processing rate of 64B/cycle.}
\label{fig:qFPGASGD}
\end{subfigure}
\quad
\begin{subfigure}[t]{.6\columnwidth}
\centering
\includegraphics[width=\columnwidth]{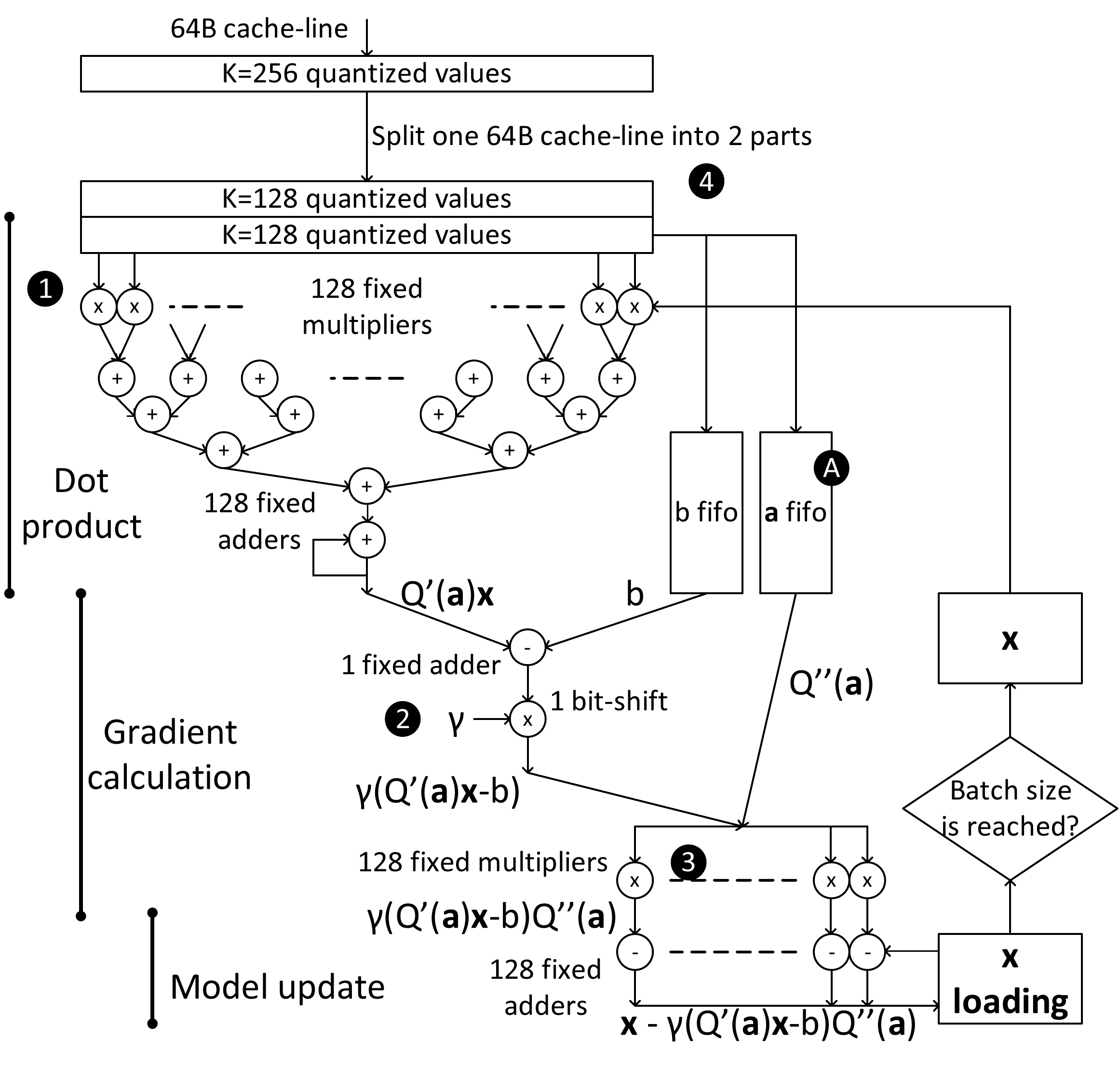}
\caption{\textit{Q1} FPGA-SGD, with a latency of 12 cycles, a data width of 32B and a processing rate of 32B/cycle.}
\label{fig:q1FPGASGD}
\end{subfigure}
\caption{Computation pipelines for all quantizations. Although for \textit{Q2}, \textit{Q4} and \textit{Q8} the pipeline width scales out and maintains 64B width, for \textit{Q1} it does not scale out and the pipeline width needs to be halved, making \textit{Q1} FPGA-SGD compute bound.}
\label{fig:qallFPGASGD}
\end{figure*}

\cleardoublepage

\bibliographystyle{icml2017}
\bibliography{low-precision.bib}

\end{document}